%% file: main.tex
\documentclass[10pt,letterpaper]{article}

\usepackage{setspace}
\usepackage{blindtext} 
\usepackage[sc]{mathpazo} 
\usepackage[T1]{fontenc} 
\usepackage{microtype} 

\usepackage[english]{babel} 

\usepackage[hang, small,labelfont=bf,up,textfont=it,up]{caption} 
\usepackage{booktabs} 

\usepackage{lettrine} 

\usepackage{enumitem} 
\setlist[itemize]{noitemsep} 

\usepackage{abstract} 

\usepackage{fancyhdr} 
\usepackage{titling} 

\usepackage{hyperref} 

\usepackage{amsmath, amsfonts, amsthm, amssymb}
\usepackage{mathpazo}
\usepackage{fullpage}
\usepackage{enumitem}
\usepackage[norelsize,algo2e]{algorithm2e}
\makeatletter
\renewcommand{\@algocf@capt@plain}{above}
\makeatother
\usepackage{setspace}
\usepackage{graphicx,tikz,pgfplots} 
\usepackage{xfrac}
\usetikzlibrary{arrows}

\usepackage{xcolor}
\usetikzlibrary{decorations.markings}
\usetikzlibrary{patterns}
\usepackage{caption, subcaption}
\usetikzlibrary{decorations.pathreplacing}
\usepackage{wrapfig}
\usepackage{tikz}
\usetikzlibrary{patterns, hobby}

\newtheorem{theorem}{Theorem}
\newtheorem{proposition}[theorem]{Proposition}

\newtheorem{claim}[theorem]{Claim}

\theoremstyle{definition}

\usepackage{soul}

\DeclareMathOperator{\argmin}{argmin}

\usepackage[noend]{algpseudocode}
\usepackage{enumitem}
\usepackage[norelsize]{algorithm2e}

\usepackage{geometry}
\usepackage{tablefootnote}

\usepackage{xfrac}
\usetikzlibrary{arrows.meta}
\usepackage{xcolor}
\usetikzlibrary{decorations.markings}
\usetikzlibrary{patterns}
\usetikzlibrary{decorations.pathreplacing}
\usepackage{wrapfig}
\usepackage{standalone}

\allowdisplaybreaks

\usepackage{multirow}
\usepackage{tablechart,booktabs,multirow,colortbl,microtype}

\usepackage[backend=biber, style=apa]{biblatex}
\usepackage{xcolor}
\usepackage{changes}
\usepackage{amsmath}
\usepackage{algorithm}
\usepackage[noend]{algpseudocode}

\title{Optimization-based Learning for Outbound Load Planning \\ in Parcel Delivery Service Networks} 
\author{%
\text{Ritesh Ojha$^{a}$, Wenbo Chen$^{a}$, Hanyu Zhang$^a$, Reem Khir$^b$, Alan Erera$^a$, Pascal Van Hentenryck$^a$} 
\\[1ex] 
\small{$^{a}$H.\ Milton Stewart School of Industrial and Systems Engineering, Georgia Tech, Atlanta, GA}\\
\small{$^{b}$School of Industrial Engineering, Purdue University, West Lafayette, IN} \\
\small{\{rojha8, wenbo.chen, hzhang747\}@gatech.edu, \{rkhir\}@purdue.edu, \{alan.erera, pvh\}@isye.gatech.edu}
}

\begin{document}
\date{}
\maketitle

\vspace{-1.2cm}

\begin{abstract}
The load planning problem is a critical challenge in service network
design for parcel carriers: it decides how many trailers (or loads),
perhaps of different types, to assign for dispatch over time between
pairs of terminals. Another key challenge is to determine a flow plan,
which specifies how parcel volumes are assigned to planned loads. This
paper considers the Outbound Load Planning Problem (OLPP) that
considers flow and load planning challenges jointly in order to adjust
loads and flows as the demand forecast changes over time before the
day of operations in a terminal. The paper aims at developing a
decision-support tool to inform planners making these decisions at
terminals across the network. The paper formulates the OLPP as a
mixed-integer programming (MIP) model and shows that it admits a large number of
symmetries in a network where each commodity can be routed through
primary and alternate terminals. As a result, an optimization solver
may return fundamentally different solutions to closely related
problems (i.e., OLPPs with slightly different inputs), confusing
planners and reducing trust in optimization. To remedy this
limitation, the first contribution of the paper is to propose a
lexicographical optimization approach that eliminates those symmetries
by generating optimal solutions staying close to a reference plan.
The second contribution of the paper is the design of an optimization
proxy that addresses the computational challenges of the optimization
model. The optimization proxy combines a machine-learning model and a
MIP-based repair procedure to find near-optimal solutions that satisfy
real-time constraints imposed by planners in the loop.  An extensive
computational study on industrial instances shows that the
optimization proxy is around $10$ times faster than the commercial
solver in obtaining solutions of similar quality; the optimization
proxy is also orders of magnitude faster for generating solutions that
are consistent with each other. The proposed approach also
demonstrates the benefits of the OLPP for load consolidation and the
significant savings obtained from combining machine learning and
optimization.
\end{abstract}


\input{introduction}

\input{literature}

\input{olpp}

\input{problem_modeling}

\input{hierarchical-model}

\input{solution_methodology}

\input{givb_heuristic}

\input{computational_study}

\section{Conclusions and Future Work}\label{sec:conclusion and future}

This paper studies the Outbound Load Planning Problem (OLPP), which
produces a load and a volume allocation plan for the commodities
processed at an individual terminal in a parcel service network. The
OLPP is intended to be solved by planners regularly in the two weeks
preceeding the day of operations. The paper was motivated by the need
of an industrial partner for a decision-support tool to advise
planners in achieving better commodity volume consolidation and
decreasing trailer costs.

The paper proposed a MIP model for the OLPP: experimental results
however revealed that the solutions of the OLPP, when solved by
commercial solvers, exhibit significant variations when commodity
volume varies. This makes the model impractical for industrial use,
i.e., with planners in the loop. To remedy this drawback, the paper
proposed the LOLPP, a lexicographic optimization model that breaks the
symmetries of the OLPP by minimizing the distance to a reference plan
and the costs of diverting flows from a primary outbound terminal
(i.e., the desired terminal chosen by default at the network level).

Although the LOLPP eliminates the instability of the OLPP and produces
consistent plans that can be used by planners, it cannot be used in
practice, as it does not meet the real-time constraints imposed by the
planner interactions. To address this computational challenge, the
paper proposed the use of optimization proxies that approximate the
LOLPP through a combination of a machine learning model and a repair
procedure that restores the feasibility of the machine learning
predictions. The paper presented the machine learning model which only
predicts the number of trailers for each outbound terminal, and the
repair procedure that adjusts the number of trailers and computes the
allocation of commodity volume to the outbound terminals. The
resulting optimization proxies were able to deliver high-quality
solutions in a fraction of the time taken by the LOLPP. An extensive
computational study on industrial instances shows that the
optimization proxy is around $10$ times faster than the commercial
solver in obtaining the same quality solutions and orders of magnitude
faster for generating solutions that are consistent with each
other. The proposed approach also highlights the benefits of the LOLPP
for load consolidation, as well as the significant savings that result
from the combination of machine learning and optimization.


This research is the first stage of a multi-stage project with our
industry partner.  Future research directions include extending the
proposed approach to clusters of terminals, taking into account their
capacities for processing commodities while tapping into the
computational efficiency of the current terminal-specific optimization
proxies. The resulting problem thus requires determining both inbound
and outbound planning decisions at each terminal, which significantly
complicates the optimization and learning models.

\section*{Acknowledgement}
This research was partly supported by the NSF AI Institute for Advances in Optimization (Award 2112533).


{{\printbibliography}}


\newpage

\section{Appendix}
\input{appendix}
\end{document}

%% file: introduction.tex
\section{Introduction}


Much of e-commerce relies on the home delivery of small parcels and other boxed freight. Leading analysts project
that today's \$3.3 trillion e-commerce market could grow further to
\$5.4 trillion annually by 2026 (\cite{WinNT}). This has led key
freight carriers like UPS and FedEx to continuously seek to redesign and
operate profitable logistic networks that meet e-commerce customer
expectations. As a result, these companies need to solve a number of
strategic, tactical, and operational optimization problems. For
instance, at the strategic level, the {\em physical network design}
must determine optimal locations and sizes of various freight
processing terminals (also known as sorting facilities). At the tactical level,
the \textit{load planning} must determine the number of trailers and
their types (also known as {\em loads}) to dispatch over time between pairs of
terminals. The resulting load plan defines the transportation capacity
of the network. In addition, at the tactical level, the {\em flow
  planning} must determine how to allocate package volumes to loads or trailers
to serve the network demand effectively. When moving from its origin
to its destination, each package is transported by a sequence of
trailers, and it gets unloaded, sorted, and loaded at each terminal while moving
between these trailers. These sequences of trailers define the {\em primary
(flow) path} for each package. Together, the flow and load plans define
a service network that moves packages from origins to destinations to
meet customer service expectations.

Tactical flow and load plans are typically generated from
\textit{average} daily estimates of origin-destination package
volume. However, in practice, package volumes may substantially change
from day to day and week to week \parencite{lindsey2016improved}.
Thus, plans based on average flows may perform poorly during actual
operations and, in some cases, may even become infeasible.  To address
this issue, planners take advantage of {\em alternate paths} for
packages, i.e., they reroute packages from primary paths to other
paths that allow for better consolidation while maintaining service
guarantees.  Such adjustments are typically performed manually:
planners use the latest available estimates on package volumes at
their terminals along with their experience to modify existing flow
and load plans. {\em The goal of this paper is to propose a systematic
  and effective framework to plan adjustments at a terminal}. The
proposed framework automates current manual processes, significantly
reducing the time planners spend on reviewing and finalizing plans
while bridging the gap between {\em tactical flow and load planning}
and {\em operational execution} for improved
performance.\footnote{Optimizing load plan adjustments at the network
  level, across all terminals simultaneously, is computationally
  challenging, especially for the large-scale networks operated by our
  industry partner; those networks have thousands of terminals and
  move millions of packages daily.}

To the best of our knowledge, this paper proposes, for the first time,
an optimization-based learning tool that is capable of generating,
{\em in a few seconds}, near-optimal adjusted load plans for a given
terminal. The proposed approach can be leveraged as part of a more
complex framework to coordinate load plan adjustments across a cluster
of terminals or an entire network. More precisely, this paper defines
the {\em Outbound Load Planning Problem} (OLPP) whose goal is to
minimize the total trailer capacity required to accommodate the
estimated package volume leaving a terminal. This OLPP decides (1) the
number of trailers and their types to be scheduled for outbound
dispatch to other terminals at different time points throughout the
day of operations, and (2) the allocation of package volumes to
available routing options while respecting trailer capacity
constraints. These two decisions define what is called a {\em load
  plan} in this paper. The OLPP is strongly NP-hard and its Mixed
Integer Programming (MIP) formulation, when solved directly using a
commercial solver, exhibits significant sensitivity to input data,
which may result in considerable variations in load planning decisions
when commodity volume varies even slightly.  This makes the model, in
its standard form, impractical for industrial use because a small
change in package volume may result in fundamentally different load
plans, which is undesirable in a planner-in-the-loop environment. It is also harder to execute given drivers availability, equipment
positioning, and other practical considerations. To address these
challenges, the paper presents an optimization-based learning
framework that delivers high-quality, stable, and feasible solutions
to OLPP in a few seconds, even for the largest terminals in one of the
largest parcel delivery service networks operated in the United
States. The proposed methodology consists of two critical steps: (1)
the derivation of a new \textit{lexicographic optimization} model that delivers
{\em stable} optimal solutions; (2) an {\em optimization proxy} that
delivers near-optimal solution to the optimization model in a few
seconds.






The main {\em methodological and business contributions} of the paper
can be summarized as follows:
\begin{enumerate}

\item The paper proposes the Lexicographic Outbound Load Planning
  Problem (LOLPP) to remedy the limitations of the standard MIP
  formulation of OLPP. The LOLPP eliminates symmetries and provides
  stable optimal load plans that are as close as possible to a
  pre-determined reference plan.

\item The paper proposes an optimization proxy to address the
  computational difficulties of the LOLPP. The optimization proxy uses
  a {\em machine learning model} to predict the number of trailers to
  dispatch for each outbound terminal and {\em a repair procedure}
  that adjusts the predictions to satisfy the problem constraints and
  determine the package volume allocation to primary and alternate
  routing options. Once trained, the optimization proxy can be
  leveraged in an online setting to provide high-quality solutions in
  a few seconds.

\item The paper presents an extensive computational study on
  industrial instances, including some of the largest terminals in a
  large-scale parcel delivery network. The results show that the optimization proxy
  significantly outperforms a greedy heuristic and the MIP formulation
  both in terms of the objective function value and
  \textit{consistency} metrics. The optimization proxy is around $10$
  times faster than the commercial solver in obtaining solutions with
  the same objective function value and orders of magnitude faster in
  terms of generating solutions that are consistent with a reference
  plan. The experimental results also show the value of breaking
  symmetries in learning accurate solution patterns and producing
  high-quality and stable load plans. The efficiency of the proposed
  approach makes it practical for real-time load plan adjustments.
  
\item From a business and sustainability perspective, the experiments
  demonstrate the value of having alternate routing options for the commodities
  in addition to the primary routing options. The proposed load plans allocate
  approximately $17\%$ package volume to the alternate routing options and
  reduce the required trailer capacity by $12\%-15\%$.
\end{enumerate}

\noindent
The remainder of this paper is organized as follows. Section
\ref{section:related} summarizes related work. Section
\ref{Sec:OLPP-all} presents a high-level description of the OLPP and
its MIP formulation, and highlights the operational planning realities
in load planning. Section \ref{sec:LOLPP} presents a MIP formulation
with a lexicographic objective function to remedy the limitations of
the MIP formulation of OLPP. Section \ref{sec: solution method}
describes the optimization proxies. Section \ref{section:heuristics}
describes a heuristic that mimics human planners and is used to
benchmark the feasible solution obtained by the optimization
proxy. Section \ref{section:Study} describes the computational
results. Section \ref{insights-and-discussions} discusses the benefits
of optimization and machine learning, quantifying the cost and
sustainability benefits and the important factors driving
them. Finally, Section \ref{sec:conclusion and future} finishes with
conclusion and discusses future research directions.

%% file: literature.tex
\section{Related Work}
\label{section:related}



Transportation operations in parcel service networks are similar to operations in the \textit{Less-than-truckload} (LTL) industry \parencite{hewitt2019enhanced}. In the LTL trucking industry, decisions have to be made at strategic, tactical, and operational levels \parencite{crainic1997planning}. Much of the related literature focuses on  service network design (SND) problems,
where different problems were classified as SND problems including flow planning problems, load planning problems, routing and dispatching problems, driver and fleet management problems, and vehicle routing and scheduling problems (\cite{bakir2021motor}).
We focus our attention on the first two problems given their direct relevance to the OLPP presented in this work.

\cite{bakir2021motor} describe the
tactical \textit{flow and load planning problem} and present a detailed
summary of the mathematical models and heuristics used to solve the problem as found in the literature. The
goal of a flow and load planning problem is to determine
a cost-effective primary path for each package and the total
trailer capacity required on each transportation arc in the service network. 
Most of these problems are formulated over time-space
networks using integer programming models. Exact approaches to solve these problems have been
proposed by \cite{boland2017continuous}, \cite{marshall2021interval},
and \cite{hewitt2019enhanced}. However, these approaches can only
solve instances with a few thousand of packages. For industry-scale
instances, researchers have resorted to various heuristics, including
variants of local search heuristic algorithms
\parencite{lindsey2016improved,erera2013improved,powell1986local}, and greedy algorithms \parencite{ulch2022greedy}.

Traditional flow and load planning problems generate a single, primary (flow) path for each package. \cite{baubaid2021value} study the value of
having alternate paths to hedge against demand
uncertainty. More specifically, the authors study the decisions that \textit{LTL} carriers need to make, in an operational setting, to effectively operate a $p-$alt load plan when demand changes dynamically on a day-to-day basis. The authors show that it is sufficient to have just one alternate option to contain the impact of most of the fluctuations in demand; such a load plan is referred to as a $2-$alt load plan. In a follow up work, \cite{baubaid2023dynamic} study a dynamic freight routing problem and propose an offline-online approximate dynamic programming solution approach for the problem. Given a specific number of shipments or pallets at a terminal in a decision epoch, the key decision in their problem is to decide the (unsplittable) routing of each shipment to a primary {\em or} an alternate arc. Along similar lines, \cite{herszterg2022near} introduce the problem of re-routing freight volume on alternate paths to improve the on-time performance of load plans on the day of operations; this becomes necessary when the actual volume deviates from the forecasted volume. In this work, package volume is unsplittable and is assigned to exactly one path such that the total (fixed) trailer capacity on each arc is respected while minimizing total shipment tardiness.
The authors formulate the problem as a MIP model and propose heuristic algorithms to solve it. It is important to note that the problem described in both \cite{baubaid2023dynamic} and \cite{herszterg2022near} considers a network of terminals and fixed trailer capacity between any pair of terminals in the network. In our problem setting, we consider outbound operations from a single terminal and have the flexibility to change the outbound trailer capacity to one or more destinations from the terminal if it leads to better consolidation while allowing volume splitting across multiple routes.
Specifically, this research bridges the gap between tactical flow and load planning \parencite{bakir2021motor} and operational execution \parencite{herszterg2022near, baubaid2023dynamic}. The goal here is to efficiently and effectively adjust the existing load plan as more accurate package volume forecast is available; this problem is mentioned as an interesting and useful future research direction by \cite{lindsey2016improved}. The flexibility to adjust the load plans enables terminal planners to better manage daily operations while maintaining service guarantees.

From a methodological perspective, we propose an optimization-based learning framework that can generate near-optimal and implementable solutions even for the largest terminal in a few seconds. In recent years, there has been a notable surge of interest among
researchers in the development of machine learning (ML) surrogates for solving MIP models. This
emerging field has attracted attention due to the potential of ML
techniques to provide efficient approximations for computationally
intensive calculations involved in solving MIPs.  We refer the reader
to (\cite{bengio2021machine, kotary2021end}) for a comprehensive
overview on the topic.  The techniques can fall into one of the two
categories.  The first category includes methods based on
reinforcement learning
(\cite{khalil2017learning,kool2018attention,fu2021generalize,yuan2022reinforcement,song2022flexible}),
where the ML model is trained by interacting with simulation
environments.  The second category comprises supervised learning
(\cite{AAAI2020,kotary2021learning,kotary2022fast,park2022confidence,chen2023two}),
where the ML model imitates the optimization model and replaces
expensive calculations with a quick approximation.  This research
focuses on the latter category since the proposed optimization model
could be used as the expert for supervised learning; we refer to this optimization-based learning methodology as the optimization
proxies. Optimization proxies, which combine learning with feasibility restoration, have
emerged from supervised learning and have been used in various application areas including energy systems and scheduling operations
(\cite{park2022confidence,kotary2022fast,chen2023two,E2ELR}). This paper adds to this stream of research and shows the potential of optimization proxies in solving logistics and transportation problems.



%% file: olpp.tex
\section{The Outbound Load Planning Problem (OLPP)} \label{Sec:OLPP-all}

This section presents a high-level problem description of the standard
OLPP problem. It then introduces a MIP formulation and discusses planning
realities in the field and the related practical challenges.

\subsection{High-Level Problem Description}
\label{sec:high-level-problem-description}

This research is conducted with a leading global parcel carrier that
operates a massive network of more than one thousand terminals and
processes millions of packages daily. As the size and volume of the
packages can be very small compared to the capacity of a trailer,
parcel carriers seek to consolidate packages from different customers
using terminals in their parcel service network (PSN).  In a PSN,
packages move from one terminal to another until they reach their
final destination.  At a specific terminal, a package has a set of
outbound terminals that it can be routed to while meeting its service
expectations; the set of such outbound terminals is determined a
\textit{priori}, in a strategic planning phase at the network level, to ensure that each
package reaches its destination in timely manner. To facilitate
consolidation, parcel carriers group packages arriving at a terminal
on a particular day into a set of commodities, where a commodity is a
grouping of packages sharing the same service class (e.g., one-day
service or two-day service), and the same destination. Each commodity
has a specific volume and that volume can be split into different
quantities that can be sent along to different outbound terminals.


This paper considers the {\em Outbound Load Planning Problem} (OLPP)
for a single terminal in a PSN. The OLPP receives a set of commodities
as input and its goal is to design a load plan for the day by
determining how many trailers are sent to each outbound terminal, and
how the volume of each commodity is split across these outbound
terminals.  The OLPP objective is to minimize trailer costs (a proxy
for minimizing operational cost), while ensuring that sufficient
trailer capacity is installed to accommodate the volume of all
commodities. Note that existing load planning literature typically
assumes that \textit{all} commodities arriving at a terminal should
be assigned to a single outbound terminal, that is, splitting is not allowed. The OLPP considered in this
paper is different: it determines how to split each commodity volume
across its outbound terminals to minimize trailer costs.

In practice, the problem is slightly more involved since each day at a
terminal is divided into time windows (typically three to four hours
in length), called \textit{sort periods} or \textit{sorts}, during
which packages are sorted. A typical operational day includes ``day'',
``twilight'', ``night'' and ``sunrise'' sorts that are non-overlapping
in time. The OLPP must be executed for multiple sorts at a terminal
and the concept of outbound terminal is replaced by a pair (terminal,
sort). The generalization of the OLPP to this more complex setting is
discussed in Appendix \ref{load_pair_idea_arcs}.

%% file: problem_modeling.tex
\subsection{Problem Formulation}
\label{sec:OLPP}


This section presents the OLPP model. Let $A$ denote the set of
outbound terminals from a given terminal $\mathcal{O}$. Let $K$ denote
the set of commodities to be sorted at terminal $\mathcal{O}$. Each
commodity $k \in K$ has a cubic volume of $q^k$ and a set of outbound
terminals $A^k \subseteq A$ that defines the next
\textit{service-feasible} terminals for the commodity; if a commodity
is routed to a service-feasible outbound terminal, then it reaches its
destination on time. The total commodity volume allocated to outbound
terminal $a \in A$ can be loaded into a set of trailer types $V_a$
where each trailer type $v \in V_a$ has a cubic capacity $Q_v$, and a
per-unit transportation cost $c_{v}$. Each outbound terminal can have
a different set of allowed trailer types, i.e., $V_{a_1}$ can be
different from $V_{a_2}$ for two different outbound terminals $a_1,a_2
\in A$. A solution of the OLPP determines the number of trailers of
each type assigned to each outbound terminal and the volume of each
commodity allocated to a trailer type for these outbound terminals. A
solution must ensure that all the volume is assigned to trailers and
that the trailer capacities are not exceeded. The objective of the
OLPP is to find a solution that minimizes total trailers' cost.

A Mixed-Integer Programming (MIP) model for the OLPP can be stated as
follows: 
\begin{subequations}\label{CreateLoadsModel}
    \begin{align}
    \underset{x,y}{\text{Minimize}} \ \quad &   \sum_{a \in A} \sum_{v \in V_a} c_{v} y_{a,v} \label{obj}\\
    \text{subject to} 
         \ \quad & \sum_{a \in A^k} \sum_{v \in V_a} x^k_{a,v} = q^k, & \quad \forall  k \in K,  \label{FlowAssignment}\\
         \ \quad & \sum_{k \in K:a \in A^k}  x^k_{a,v}\leq Q_v \, y_{a,v}, & \quad \forall  a \in A, v \in V_a,  \label{LoadCapacity}\\
         \ \quad & x^k_{a,v} \geq 0 & \quad \forall k \in K, a \in A^k, v \in V_a,  \label{Nonneg1} \\
         \ \quad & y_{a,v} \in \mathbb{Z}_{\geq 0} & \quad \forall a \in A, v \in V.  \label{Nonneg2}
    \end{align}
\end{subequations}

\noindent
It uses a non-negative continuous decision variable $x^k_{a,v}$ to
represent the volume of commodity $k \in K$ allocated to trailer type
$v \in V_a$ planned for outbound terminal $a \in A^k$, and an integer
decision variable $y_{a,v}$ to determine the number of trailers of
type $v$ planned for outbound terminal $a \in A$. The experiments use $c_v =
Q_v \ \forall \ v \in V$ in the objective \eqref{obj}, i.e., the model
minimizes the total trailer capacity required to accommodate the commodity
volume. Constraints \eqref{FlowAssignment} ensure that the volume of
each commodity is assigned to outbound terminals. Constraints
\eqref{LoadCapacity} ensure that the total volume allocated to an outbound terminal
respects the installed trailer capacity associated with the
terminal. Constraints \eqref{Nonneg1}-\eqref{Nonneg2} define the
domain and range of variables.

It is worth noting that OLPP is conceptually related to the
variable-sized bin-packing problem described by
\cite{friesen1986variable}, where the objective is to minimize the
total space used to pack a set of items into bins (available in
different sizes), such that each item is packed into exactly one
bin. In the OLPP, the packages are the items and trailers are bins;
The key difference is that the OLPP allows for the {\em splitting} of
the package volume across multiple trailers (determined by the primary
and alternate terminals); this makes it possible to reduce the
transportation cost by promoting better consolidation or packing.

Appendix \ref{appendix:complexity} includes complexity results and
related proofs for the OLPP problem. The OLPP, in its general form, is
strongly NP-hard. It becomes weakly NP-hard when each commodity is
compatible with exactly one or all outbound terminals, and there are
multiple trailer types. It becomes polynomial when each commodity is
compatible with exactly one or with all outbound terminals, and there
is only one trailer type.

\subsection{Planning Realities in the Field}
\label{sec:planning-realities-in-the-field}

The OLPP model presented in the previous section is idealized. This
section reviews planning realities in the field, which require
generalizations of the OLPP.

\paragraph{Primary and Alternate Routing Options}

Parcel carriers typically identify a set of outbound terminals for
each commodity processed at the current terminal; the preferred option
is the {\em primary} outbound terminal and the other options are
referred to as {\em alternate} outbound terminals. The choice of the
alternate terminals is usually determined after identifying the
primary terminals, typically using historical data and past
experiences \parencite{baubaid2021value}. It is desirable for a
practical OLPP model to send commodities to their primary outbound
terminal as much as possible, while using alternate terminals as
opportunities for better consolidations. Figure
\ref{fig:rules_terminals_paths} illustrates the concepts of primary
and alternate terminals. Recall that the primary and alternate
outbound terminals for each commodity are defined in such a way that
the commodities allocated to these routing options reach their
destination on time.

\begin{figure}[ht!]
    \centering
    \includegraphics[width=8cm, height=5cm]{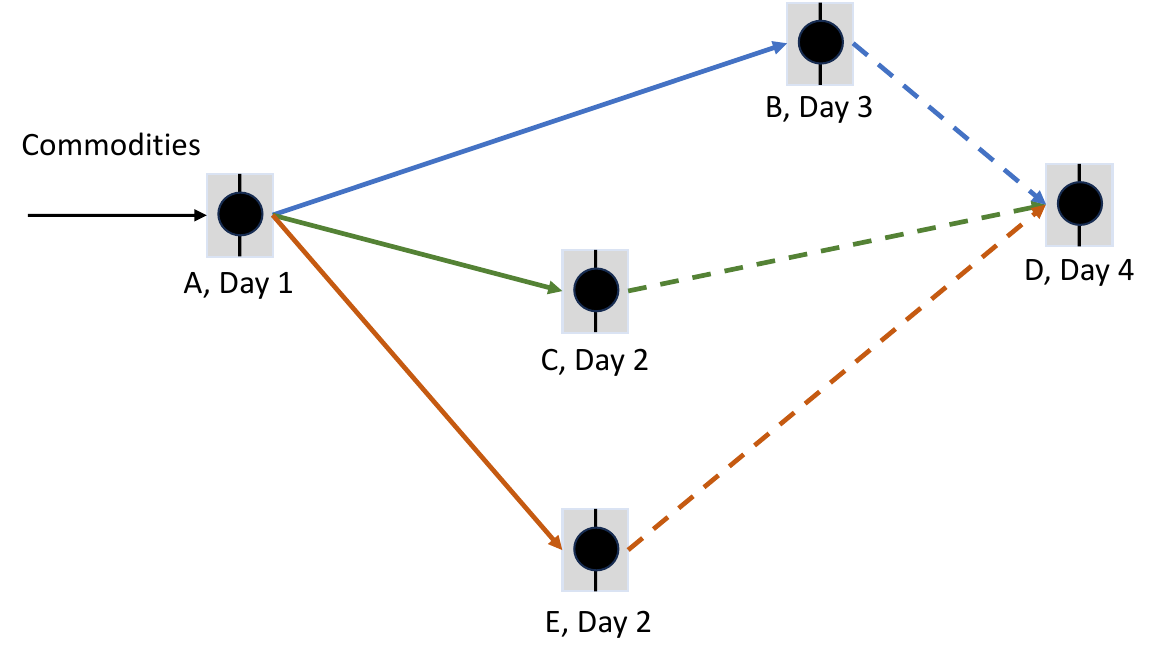}
    \caption{The \textit{primary} terminal for commodities processed at terminal A and
      destined for terminal D on day $4$ is denoted by solid green
      arc and the two alternate terminals are denoted by solid blue
      and solid brown arcs.}
    \label{fig:rules_terminals_paths}
\end{figure}


\paragraph{Planners in the Loop}

This research is conducted with a leading global parcel carrier, with
planners in the loop in most terminals.  The OLPP is solved
periodically over a period going from the two weeks prior to the day
of operations to the day of operations. For each time period, each
terminal in the network has a set of forecasted inbound commodities.
The role of the planners is to adjust an existing flow and load plan
to accommodate the changes in forecasted inbound commodities.  When
automating the planning process with optimization and/or machine
learning, it is thus critical to present planners with plans that are
consistent over time as much as possible.

Unfortunately, the OLPP model exhibits a large number of variable
symmetries, i.e., the values of several variables can be permuted to
obtain another optimal solution. Consider the example shown in Figure
\ref{sym123} that illustrates a \textit{primary terminal} $(C,Day 2)$,
an alternate terminal $(B,Day 3)$, and an alternate terminal $(E,Day
2)$ for a set of commodities $K$ processed at terminal A on day $1$
and destined for terminal D on day $4$. Figure \ref{sym123} depicts a
simple instance with symmetric solutions that are operationally
different from one another, yet they are equivalent, from the OLPP
Model (\ref{CreateLoadsModel}) perspective, as they require the
same number of trailers of the same type, and hence the same total
trailer capacity. The values of the $y$ variables for each outbound
terminal in Figure \ref{sym1_new} can be permuted to get the solution
in Figures \ref{sym2_new} and \ref{sym3_new} without changing the
problem structure. The alternate optimal solutions exist in the OLPP
Model (\ref{CreateLoadsModel}) because commodities are
indifferent to the primary and alternate terminals they are assigned
to, as the volume allocation decisions ($x$ variables) do not incur
any cost.

\begin{figure}[t!]
    \begin{subfigure}[t]{0.33\textwidth}
        \includegraphics[width=\textwidth]{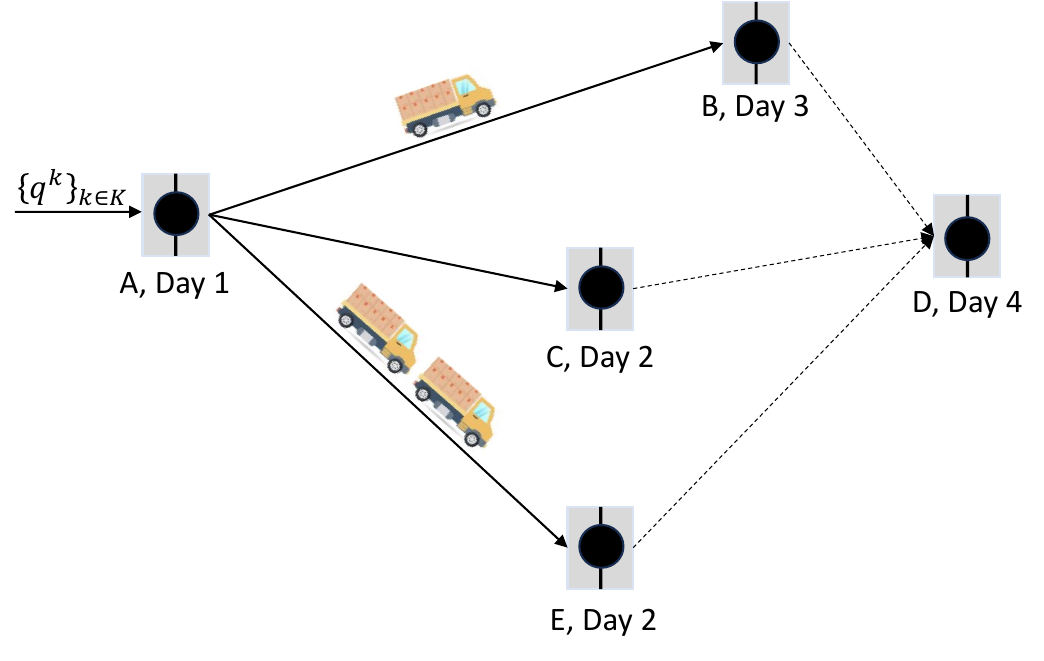}
        \caption{} \label{sym1_new}
     \end{subfigure}%
     \hfill
    \begin{subfigure}[t]{0.33\textwidth}
        \includegraphics[width=\textwidth]{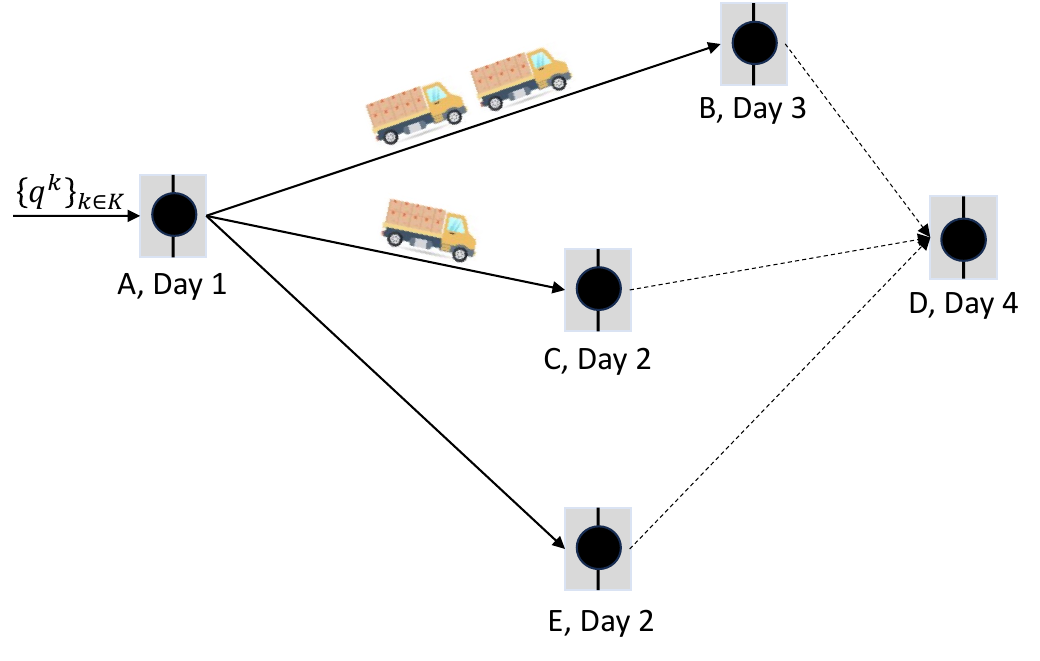}
        \caption{} \label{sym2_new}
     \end{subfigure}%
     \hfill
    \begin{subfigure}[t]{0.33\textwidth}
        \includegraphics[width=\textwidth]{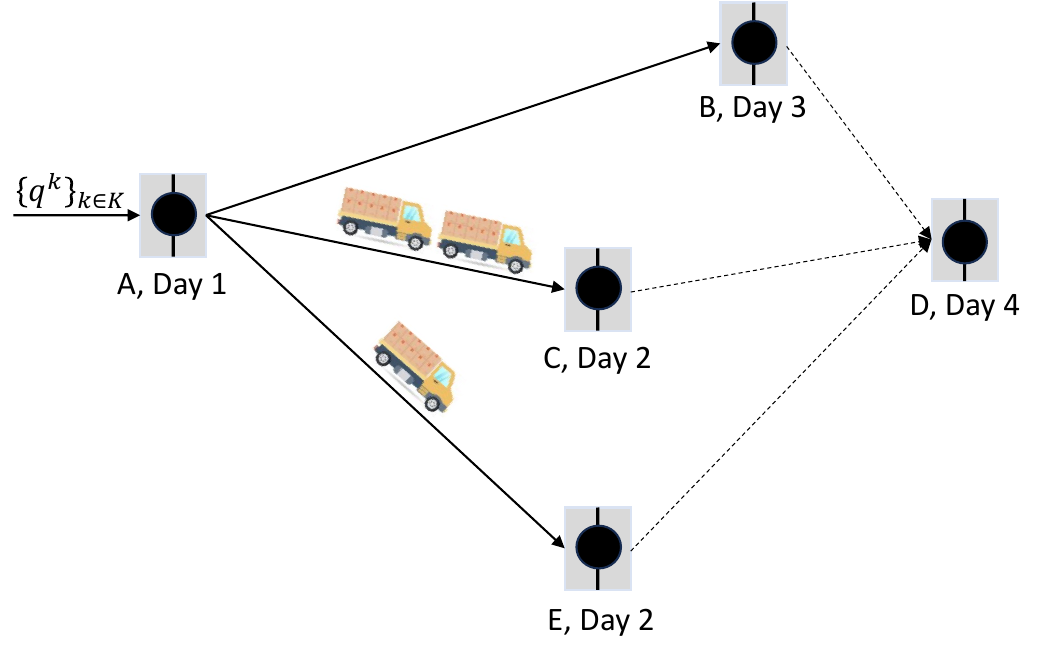}
        \caption{} \label{sym3_new}
     \end{subfigure}%
     \hfill
     \caption{An example to highlight the alternate (symmetric) optimal solutions for OLPP (Model \ref{CreateLoadsModel}).}
\label{sym123}
\end{figure}

Such symmetries are undesirable for many reasons. Paramount among them
are the realities in the field: the model is intended to be used and
validated by planners. Furthermore, the OLPP is sensitive to minor
changes in input data (commodity volume). If small or no variations of
total commodity volume produce fundamentally different solutions,
planners are unlikely to trust the model.  Indeed, since the model is
used multiple times during the day, it is important to ensure that
successive optimal solutions are as consistent with each
other as possible. For the example shown in Figure \ref{sym1_new}, consider two
instances with commodity volume $\phi_1$ at time $t_1$ and commodity
volume $\phi_2 = \phi_1 + \epsilon$ at time $t_2$. If both instances
have optimal solution with three trailers, it is desirable that they
would allocate trailers similarly. However, the OLPP model, when
solved by a commercial solver, does not guarantee this, highlighting
the sensitivity of the model to input data.

%% file: hierarchical-model.tex
\section{The Lexicographic Outbound Load Planning Problem (LOLPP)}
\label{sec:LOLPP}

This section presents the Lexicographic Outbound Load Planning Problem
(LOLPP) designed to account for the realities in the field. The LOLPP
uses a lexicographic objective function to produce stable plans
and break symmetries, while optimizing for total trailer capacity.
The LOLPP is represented by Model \ref{Stage2Model}:
\begin{subequations}\label{Stage2Model}
    \begin{align}
    \underset{x,y}{\text{Lex-Minimize}} \ \quad &  \langle \sum_{a \in A} \sum_{v \in V_a} c_{v} y_{a,v} ,\sum_{a \in A} \sum_{v \in V_a}\vert y_{a,v} - \gamma_{a,v} \vert,\sum_{k \in K}\sum_{a \in A_k}\sum_{v \in V_a}d^k_a x^k_{a,v} \rangle \label{obj_s2}\\
    \text{subject to} 
         \ \quad & \sum_{a \in A^k} \sum_{v \in V_a} x^k_{a,v} = q^k, & \quad \forall  k \in K,  \label{FlowAssignment_s2}\\
         \ \quad & \sum_{k \in K:a \in A^k}  x^k_{a,v}\leq Q_v (y_{a,v}), & \quad \forall  a \in A, v \in V_a,  \label{LoadCapacity_s2}\\
         \ \quad & x^k_{a,v} \geq 0 & \quad \forall k \in K, a \in A^k, v \in V_a,  \label{Nonneg1_s2} \\
         \ \quad & y_{a,v} \in \mathbb{Z}_{\geq 0} & \quad \forall a \in A, v \in V.  \label{Nonneg2_s2}
    \end{align}
\end{subequations}
Model \ref{Stage2Model} receives two key inputs: (1) a
\textit{reference plan} $\gamma$, i.e., the last available tactical
outbound load plan for the terminal under consideration that specifies
the number of trailers of different types planned to operate for each
outbound terminal; and (2) the costs of allocating commodity volume to
its primary and/or alternate terminals. Its objective function
leverages these two inputs to account for the realities in the
field. More precisely, the objective function \eqref{obj_s2} is
lexicographic and consists of three sub-objectives that are optimized
in sequence.

The first sub-objective is the objective of the OLPP: it minimizes the
total trailer cost. The second sub-objective minimizes the deviation
of the solution from the reference plan $\gamma$. The third
sub-objective minimizes the {\em flow diversion cost}.  In the second
sub-objective, $\gamma_{a,v}$ denotes the total number of trailers of
type $v \in V_a$ assigned to terminal $a \in A$ in the reference plan;
the goal is to minimize the L1-norm between the solution $y$ and
$\gamma$. In the third sub-objective, $d^k_a$ denotes the cost of
allocating a unit of commodity $k \in K$ to a terminal $a \in A^k$;
recall that $A^k$ denotes the set of primary and alternate outbound
terminals for commodity $k \in K$. The allocation cost to the primary
outbound terminal is zero, as it represents the desired routing option. The
cost of allocating a commodity volume to an alternate terminal is
proportional to the distance between the alternate terminal and the
final commodity destination, since it is preferable to route
commodities as close as possible to their ultimate destination. The
purpose of this diversion cost is to break symmetries between
solutions with the same distance to the reference plan. For instance,
it breaks the symmetries in Figure \ref{sym123}, as the cost $d^k_a$
is different for different outbound terminals $a$ for a given
commodity $k$. Note also that constraints \eqref{FlowAssignment_s2},
\eqref{LoadCapacity_s2}, \eqref{Nonneg1_s2} and \eqref{Nonneg2_s2} are
the same as in the OLPP Model (\ref{CreateLoadsModel}).

The LOLPP uses the second and the third sub-objectives to account for
some of the network effects, so that the generated load plans are not
completely myopic in nature. Note that the reference plan $\gamma$ for
a given terminal is derived from a tactical load plan developed at the
network level. Therefore, the second sub-objective ensures that the
optimized load plan for the terminal does not deviate significantly
from the network-level load plan. Furthermore, if it becomes necessary
to allocate the volume of a commodity to an alternate outbound
terminal, then the flow diversion cost (third sub-objective)
prioritizes the alternate terminals that are closer to the final
destination as opposed to choosing the nearest alternate outbound
terminal, which can be a myopic decision.

Figure \ref{fig:high_sensitivity} shows the benefits of the LOLPP: it
illustrates the sensitivity of the trailer decisions, i.e., the values
of $y$-variables for some individual outbound terminals when the
total commodity volume $\sum_{k \in K}q^k$ varies. The vertical axis
denotes the number of trailers of a specific type used for a specific
outbound terminal, and the horizontal axis denotes the total commodity
volume. The red plot represents decisions from solving the OLPP Model
(\ref{CreateLoadsModel}), the orange plot represents decisions of the
LOLPP Model (\ref{Stage2Model}), and the blue plot represents
decisions of the LOLPP without the second sub-objective
LOLPP$^{-2}$. 

\begin{figure}[!t]
    \begin{subfigure}[t]{0.5\textwidth}
        \includegraphics[width=\textwidth]{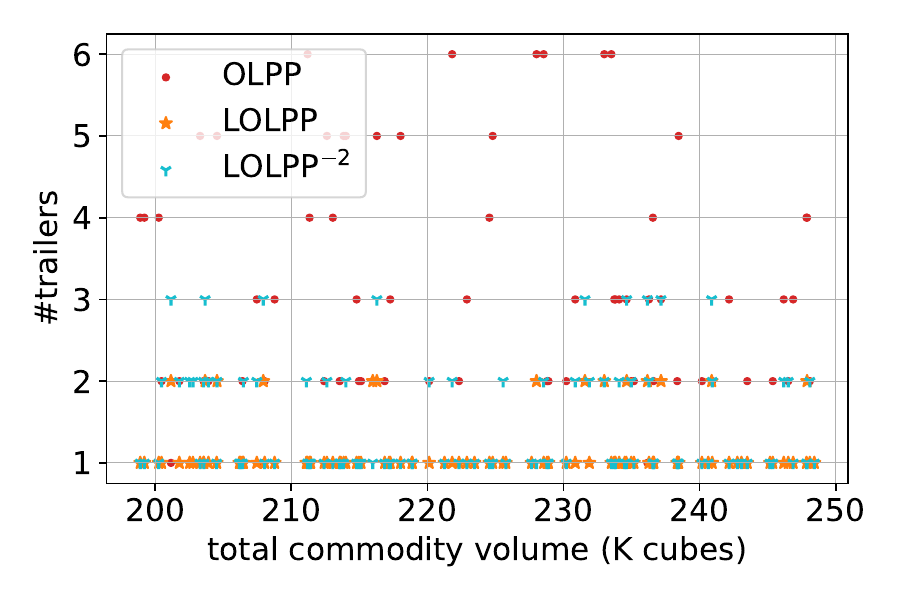}
        \caption{Outbound Terminal 1}
     \end{subfigure}%
     \hfill
    \begin{subfigure}[t]{0.5\textwidth}
        \includegraphics[width=\textwidth]{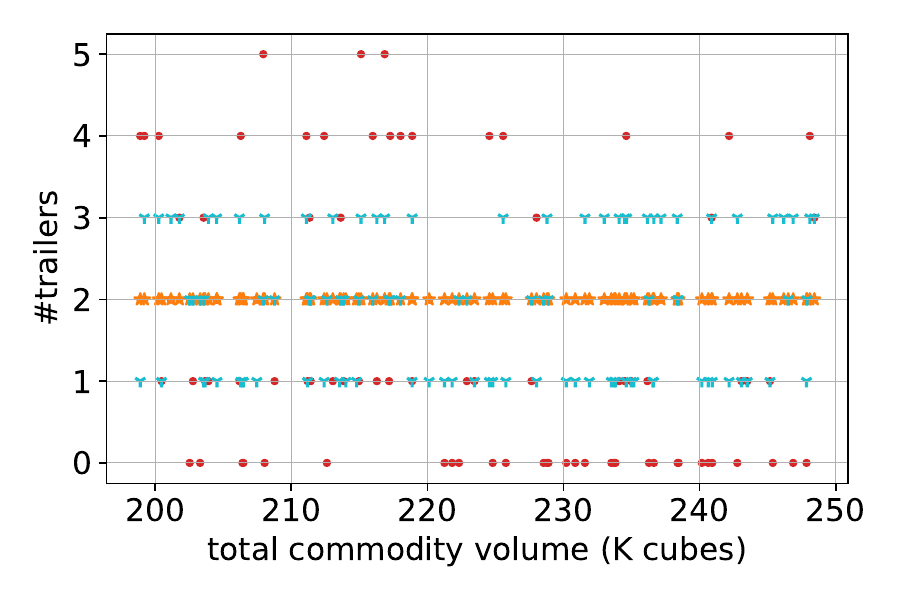}
        \caption{Outbound Terminal 2}
     \end{subfigure}%
     \hfill
    \caption{(color online) Sensitivity analysis of the OLPP, LOLPP, and LOLPP$^{-2}$ on a real-world medium-sized instance from our industry partner. } \label{fig:high_sensitivity} 
\end{figure}


In the illustrated examples, as the total commodity volume increases, the OLPP
model produces solutions where the trailer decisions fluctuate
dramatically between $1$ and $6$ trailers for outbound terminal $1$ and
between $1$ and $5$ trailers for outbound terminal $2$. The trailer
decisions of the LOLPP model are more consistent, and vary between $1$
and $2$ trailers for outbound terminal $1$, and is constant at $2$
trailers for outbound terminal $2$. When the second sub-objective is
dropped, the trailer decisions vary between $1$ and $3$ for both
outbound terminals, showing the effectiveness of the symmetry-breaking
approach. It is worth noting that, in the OLPP, the \textit{total} trailer
capacity is a non-decreasing function of the total commodity
volume. However, the total trailer capacity planned for \textit{an
  outbound terminal} may increase or decrease when the total commodity
volume increases, as seen in Figure \ref{fig:high_sensitivity}.

%% file: solution_methodology.tex
\section{Learning the LOLPP}
\label{sec: solution method}

The LOLPP model is typically too slow to be of practical use in
real-time environments with planners in the loop. This section proposes an optimization-based ML approach to address these computational limitations.

\subsection{The Overall Approach}

The LOLPP can be viewed as a function $f$ that, given a vector of
commodities $\mathbf{q}$, returns an optimal solution
$(\mathbf{y}^*,\mathbf{x}^*)$ specifying the load plan and the
allocation of commodity volume to the primary and alternate terminals,
respectively. A corresponding machine learning model is a function
$\hat{f}$ that approximates $f$, i.e., $f(\mathbf{q}) \approx
\hat{f}(\mathbf{q})$.  More precisely, the machine learning model
$\hat{f}$, given the same vector of commodities $\mathbf{q}$, should
return an approximate solution $(\hat{\mathbf{y}},\hat{\mathbf{x}})$
such that $\hat{y}_{a,v} \approx y^*_{a,v}$ and $\hat{x}^k_{a,v}
\approx x^{k*}_{a,v}$. However, such approximation does not guarantee
the feasibility of generated solutions; the difficulty comes from the
constraints of the LOLPP, i.e., constraints
\eqref{FlowAssignment_s2}--\eqref{LoadCapacity_s2}, which are highly
likely to be violated by the machine learning model. In other words,
it is unlikely that $ (\hat{\mathbf{y}},\hat{\mathbf{x}})$ be a
feasible solution to the LOLPP.  This paper explores the concept of
{\em optimization proxy} to remedy this issue and produce
high-quality, feasible solutions for the LOLPP.

An optimization proxy is composed of a traditional machine learning
model that produces an approximate solution, say
$(\hat{\mathbf{y}},\hat{\mathbf{x}})$ in the LOLPP, and a repair layer
that transforms the approximation into a feasible solution, say
$(\bar{\mathbf{y}},\bar{\mathbf{x}})$, of the LOLPP. See, for
instance,
\cite{SIAMNEWS,AAAI2020,AAAI2023,IEEETPS2022spatial,PSCC2022,kotary2021learning,kotary2022fast,park2022confidence,kotary2022fast,chen2023two,E2ELR}
for an overview of optimization proxies and their applications.

\subsection{Optimization Proxies}

Figure \ref{fig:ml:pipeline} depicts the pipeline of the
optimization-based learning approach used in this paper. At
learning time (offline), the approach generates a dataset of instances, i.e.,
mappings between inputs and outputs of the LOLPP. These instances are
used to train a machine learning model. At inference time (online)
when planners use the model, the machine learning model is used to
obtain an approximated solution to an unseen instance (i.e., a set of
commodities for the LOLPP) which is then transformed into a feasible
solution by the repair layer, i.e., a feasible load plan and a
commodity volume allocation plan for the LOLPP.

\begin{figure}[!t]
        \includegraphics[width=\textwidth]{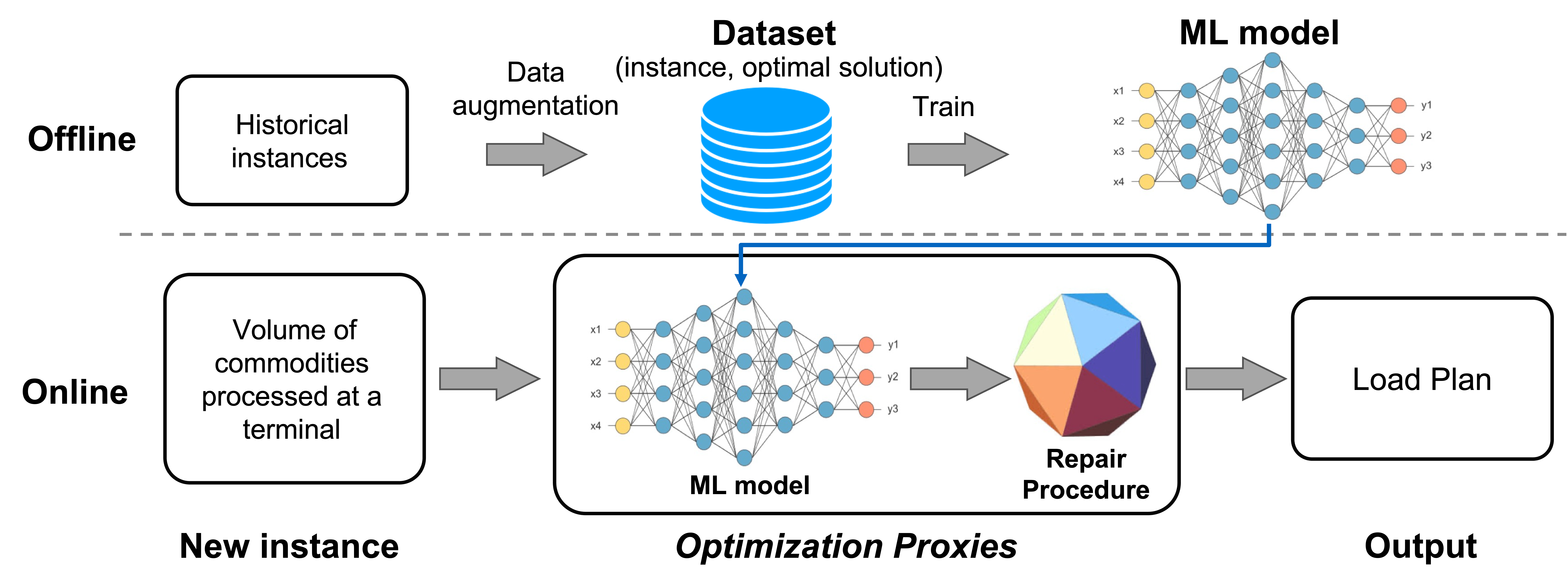}
        \caption{The Overall Pipeline of the Optimization-based Learning
          Approach: In the offline setting, the dataset is generated
          and augmented by solving historical instances with the
          proposed LOLPP model. The ML model is trained in a
          supervised learning fashion. In the online setting, the
          optimization proxies take as input a new instance and output
          a near-optimal feasible load plan and commodity volume allocation plan in seconds.}
        \label{fig:ml:pipeline}
\end{figure}

A simple way to obtain a feasible solution to LOLPP is to predict the commodity volume allocation ($\hat{x}$) variables and determine a corresponding feasible trailer decision ($\Bar{y}$) variables as shown below
\[
\bar{y}_{a,v}=\left \lceil \frac{\sum_{k \in K:a \in A^k} \hat{x}^k_{a,v}}{Q_v} \right \rceil \ \forall a \in A, v \in V_a.
\]
However, the feasible solution obtained from the
approach can be rather poor in terms of the total trailer capacity
because the trailer decisions can be very sensitive to the predicted commodity
volume allocation decisions. Consider, for example, an optimal solution where 100
cubic volume is allocated to an outbound terminal, which requires two
trailers, each with a capacity of 50 units.  If the machine learning model predicts
a volume of $100.5$, then the total number of trailers required
will be $\left \lceil \frac{100.5}{50} \right \rceil = 3$, which
generates a poor solution of the LOLPP.

To remedy this problem, this paper focuses on a machine learning model
that predicts the optimal trailer decisions. The repair step then
adjusts these trailer decisions to obtain a feasible solution to the
LOLPP, including the allocation decisions of commodity volume to their
respective primary and/or alternate outbound terminals.


\subsection{The Machine Learning Model}
\label{sec:machine-learning-model}

This section presents the machine learning model that, given the commodity
volumes, predicts the trailer decisions for every outbound
terminal. Consider a machine learning model $\hat{f}_{\theta}$,
parameterized by $\theta$, that maps its inputs, i.e., the commodity
volume, to the optimal trailer decisions: \eqref{map1}-\eqref{map2}.
\begin{subequations}
\begin{align}
        \hat{f}_{\theta}:&\mathbb{R}_{\geq 0}^{|K|} \longrightarrow \mathbb{Z}_{\geq 0}^{|A| \times |V|} \label{map1}\\
        & \mathbf{q} \longmapsto \mathbf{y}\label{map2}
\end{align}
\end{subequations}

\noindent
The machine learning model $\hat{f}_{\theta}$ is trained using a
dataset of instances $\{(\mathbf{q}_i, \mathbf{y}^*_i)\}_{i \in N}$,
where $N$ is the set of instances, $\mathbf{q}_i$ is the input of
instance $i$, and $\mathbf{y}^*_i$ is the optimal trailer decision of
instance $i$. The best parametrization $\theta^*$ can be obtained by
minimizing the empirical risk
\begin{subequations}
\label{eq:ml:formulation}
\begin{align} 
\theta^* = \argmin_{\theta} \quad & \frac{1}{|N|} \sum_{i\in N} \; {\cal L}(\hat{f}_{\theta}(\mathbf{q}_i),\mathbf{y}_i^*) \label{eq:emprirical risk}
\end{align}
\end{subequations}
where ${\cal L}$ is the loss function (e.g., the L1-distance between the predicted
and optimal trailer decisions).

The machine learning model used in this paper is a deep neural network as
illustrated in Figure \ref{fig:ml:architecture}.  It consists of a
Multi-Layer Perceptron (MLP) and a series of operators such as
reshaping, masking, and rounding.  The MLP comprises multiple layers,
where each layer includes a dense perceptron, a batch normalization
(\cite{ioffe2015batch}), a dropout component
(\cite{srivastava2014dropout}), and a ReLU (Rectified Linear Unit)
activation function.  

\begin{figure}[!t]
    \includegraphics[width=\textwidth]{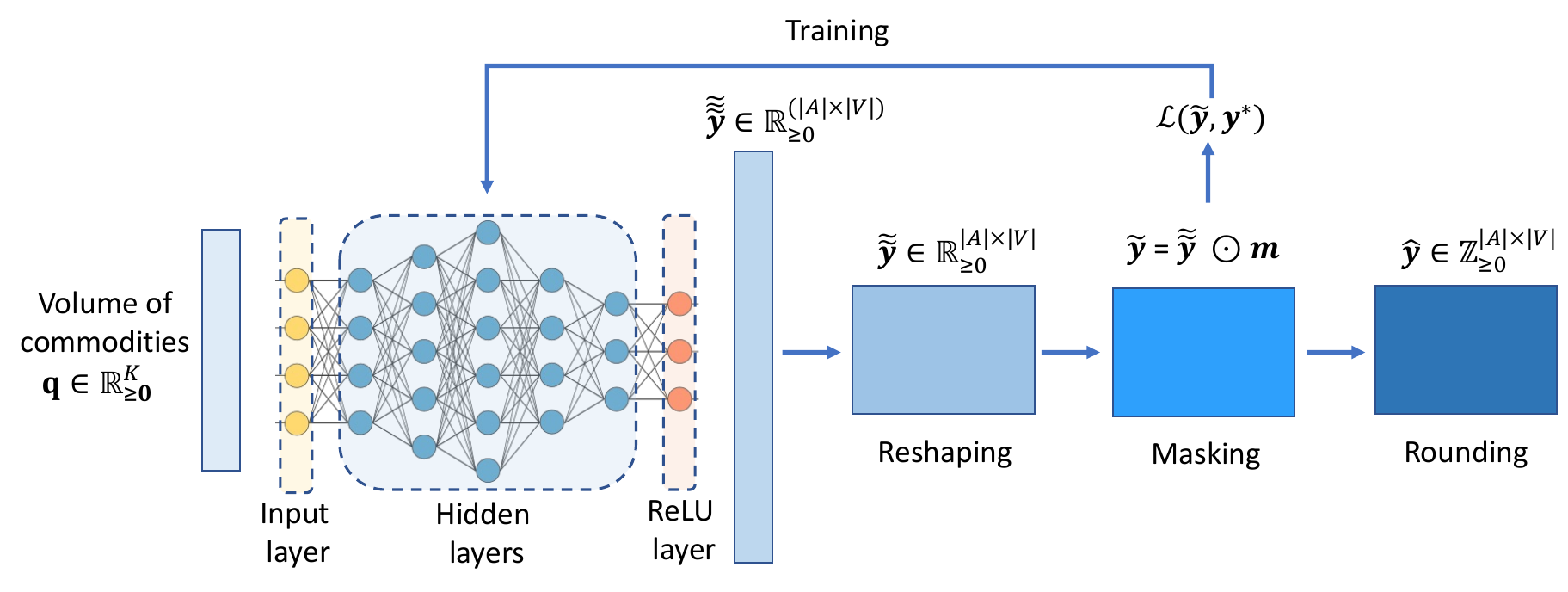}
    \caption{The Machine Learning Model for the LOLPP.}
    \label{fig:ml:architecture}
\end{figure}

The MLP maps the input parameter $\mathbf{q}$ to
a non-negative trailer decision vector
$\Tilde{\Tilde{\Tilde{\mathbf{y}}}}$ with dimension $(|A|\times|V|)$.
The reshaping operator converts the vector into a matrix
$\Tilde{\Tilde{\mathbf{y}}} \in \mathbb{R}_{\geq 0 }^{|A|\times |V|}$,
where entry $\Tilde{\Tilde{y}}_{a,v}$ could be interpreted as the
approximated number of trailers with equipment type $v \in V$ for
outbound terminal $a \in A$.  The masking operator sets the trailer
decisions on the incompatible terminals to be zero by computing
the Hadamard product $\Tilde{\mathbf{y}} = \Tilde{\Tilde{\mathbf{y}}}
\odot \mathbf{m}$, where $\mathbf{m}$ denotes the binary compatibility
matrix with $m_{a, v} = 1$ indicating that equipment type $v \in V$ is
compatible with terminal $a \in A$ and $m_{a,v} = 0$ otherwise.

The model is trained with a smoothed-$L_1$ loss, 
\begin{equation*}
    {\cal L}(\Tilde{\mathbf{y}}, \mathbf{y}^*) = \sum_{a\in A, v \in V} 0.5(\Tilde{y}_{a,v} - y^*_{a,v})^2 \mathbb{1}_{\{|\Tilde{y}_{a,v} - y^*_{a,v}| < 1\}} + (|\Tilde{y}_{a,v} - y^*_{a,v}| - 0.5)\mathbb{1}_{\{|\Tilde{y}_{a,v} - y^*_{a,v}| \ge 1\}}, 
\end{equation*}
where $\mathbb{1}$ denotes the indicator function.  The gradient of
the loss is used to update the parameters of the MLP using stochastic
gradient descent (\cite{kingma2014adam}) with backpropagation
(\cite{rumelhart1986learning}).  At inference time, the rounding
operator ensures that the output trailer decisions are integers.

\subsection{The Repair Procedure}

At inference time, the machine learning model produces an
approximation to the number of trailers $\hat{y}_{a,v}$ allocated
to each outbound terminal $a \in A$ and equipment type $v \in
V_a$.
A commodity volume allocation plan can be obtained by solving the following
set of equations:
\begin{subequations}\label{feasibility_lp}
    \begin{align}
      \quad & \sum_{a \in A^k} \sum_{v \in V_a} \hat{x}^k_{a,v} = q^k, & \quad \forall  k \in K,  \label{inf_volume}\\
               \ \quad & \sum_{k \in K:a \in A^k}  \hat{x}^k_{a,v}\leq Q_v (\hat{y}_{a,v}), & \quad \forall  a \in A, v \in V_a,  \label{inf_capacity}\\
         \quad & \hat{x}^k_{a,v} \geq 0 & \quad \forall k \in K, a \in A^k, v \in V_a.  \label{inf_Nonneg1}
    \end{align}
\end{subequations}
There is however no guarantee that this set of equations admits a solution given the prediction $\hat{y}$.

The repair layer for the LOLPP learning is a two-step procedure. In
the first step, the repair procedure aims at identifying a small set
of outbound terminals for which additional trailers would remove the
infeasibilities. In a second step, the repair procedure minimizes the
number of additional trailers required, while
ensuring the feasibility of the resulting load plan.

The first step is a linear program that minimizes capacity violations for the outbound terminals:
\begin{subequations}
\label{feasibility_recov_lp}
\begin{align}
    \underset{\hat{x}, z}{\text{Minimize}} \ \quad & \sum_{a \in A,v \in V_a} z_{a,v} \label{fr_obj}\\ 
\text{subject to} 
    \ \quad & \sum_{a \in A^k} \sum_{v \in V_a} \hat{x}^k_{a,v} = q^k, & \quad \forall  k \in K,  \label{fr_FlowAssignment}\\
    \ \quad & \sum_{k \in K:a \in A^k}  \hat{x}^k_{a,v}\leq Q_v (\hat{y}_{a,v}) + z_{a,v}, & \quad \forall  a \in A, v \in V_a,  \label{fr_LoadCapacity}\\
         \ \quad & \hat{x}^k_{a,v},z_{a,v} \geq 0 & \quad \forall k \in K, a \in A^k, v \in V_a,  \label{fr_Nonneg1}
    \end{align}
\end{subequations}
The continuous decision variables $z_{a,v} \geq 0$ denote the capacity
violation on outbound terminal $a \in A$ and equipment $v \in
V_a$. The objective function \eqref{fr_obj} minimizes the total
capacity violations. Constraints \eqref{fr_FlowAssignment} ensure that
the total volume of every commodity is assigned to its primary and/or alternate outbound
terminal. Constraints \eqref{fr_LoadCapacity} determine the capacity
violations for each outbound terminal and equipment type. Constraints
\eqref{fr_Nonneg1} define the domain and range of variables.  If Model
\ref{feasibility_recov_lp} has an optimal objective value equal to
$0$, then it has found a feasible solution to LOLPP. Otherwise, Model
\ref{feasibility_recov_lp} has identified a set of pairs (outbound
terminal, equipment type) for which additional trailer capacity would
make it possible to recover a feasible solution; that is, those with $z^*_{a,v} > 0$ in the optimal solution $\textbf{z}^*$.


The second step solves a MIP model to produce a feasible solution to
the LOLPP. The key idea of the MIP model is to allow additional
trailer capacity where capacity violations have been detected, while
keeping the remaining feasible predicted loads unchanged. This should reduce
significantly the size of the MIP model, making the repair step much
more efficient than the LOLPP model. Let $z^*$ be the optimal solution of Model
\eqref{feasibility_recov_lp} and define
\begin{align}
  AV = & \{ (a,v)  \mid  a \in A, v \in V_a \}, \\
  AV^+ = & \{ (a,v) \in AV  \mid  z^*_{a,v} > 0 \}.
\end{align}
The number of additional units of trailers $\xi_{a,v}$ of type $v \in V_a$ available to each outbound terminal in $a \in A$ is given by the
formula
\[
\xi_{a,v} = \left \lceil{\frac{z^*_{a,v}}{Q_v}}\right \rceil.
\]
Note that adding these capacities are sufficient to find a feasible
solution to the LOLPP. The goal of the MIP model is to determine the
best combination of additional capacities to add in order to minimize
costs by leveraging consolidation opportunities. Consider for instance
a set of commodities, all of which can be allocated to terminals $a_1$
and $a_2$ and a trailer $v$ with a capacity of $2$ units. Suppose the
optimal solution of Model (\ref{feasibility_recov_lp}) is
$z_{a_1,v}=z_{a_2,v}=1$. In this case, a feasible OLPP solution to
Model (\ref{CreateLoadsModel}) can be recovered by adding two
trailers, one for each terminal. However, there exists a LOLPP
solution with only one trailer on any one of the two outbound
terminals.  This can be found using the second step of the repair
procedure, represented by the following MIP model.
\begin{subequations}\label{feasibility_mip}
    \begin{align}
    \underset{u, x}{\text{Minimize}} \ \quad & \sum_{(a,v) \in AV^+} u_{a,v} \label{fr_obj_mip}\\ 
    \text{subject to} 
    \ \quad & \sum_{a \in A^k} \sum_{v \in V_a} x^k_{a,v} = q^k, & \quad \forall  k \in K,  \label{fr_FlowAssignment_mip}\\
    \ \quad & \sum_{k \in K:a \in A^k}  x^k_{a,v}\leq Q_v \ \hat{y}_{a,v}, & \quad \forall  (a,v) \in AV \setminus AV^{+},  \label{fr_LoadCapacity1_mip}\\
    \ \quad & \sum_{k \in K:a \in A^k}  x^k_{a,v}\leq Q_v  \ (\hat{y}_{a,v} + u_{a,v})  & \quad \forall  (a,v) \in AV^{+},  \label{fr_LoadCapacity2_mip}\\
    \ \quad & x^k_{a,v} = \hat{x}^k_{a,v} & \quad \forall k \in K, (a,v) \in AV \setminus AV^{+},  \label{fixed_flows} \\
         \ \quad & x^k_{a,v}\geq 0 & \quad \forall k \in K, a \in A^k, v \in V_a,  \label{fr_Nonneg1_mip} \\
         \ \quad & u_{a,v} \in \{0,\dots,\xi_{a,v}\} & \quad \forall (a,v) \in AV^+.  \label{fr_Nonneg2_mip}
    \end{align}
\end{subequations}

\noindent The objective function in \eqref{fr_obj_mip} minimizes the
additional trailer capacity added on pairs $(a,v) \in AV^+$ with
insufficient capacity. Constraints \eqref{fr_FlowAssignment_mip}
ensure that the commodity volume is assigned to the primary and/or
alternate terminals. Constraints \eqref{fr_LoadCapacity1_mip} and
\eqref{fr_LoadCapacity2_mip} ensure that the commodity volume
allocated to each pair $(a,v)$ respects the trailer capacities:
constraint \eqref{fr_LoadCapacity2_mip} have the flexibility of adding
additional capacity, while Constraints \eqref{fr_LoadCapacity1_mip}
use the trailer predictions. Constraints \eqref{fr_Nonneg1_mip} and
\eqref{fr_Nonneg2_mip} define the domain and range of variables.
Constraint \eqref{fixed_flows} can be used to fix the commodity volume
allocation to the outbound terminals with sufficient predicted
capacity of trailer type $v$ (i.e., for outbound terminals $AV
\setminus AV^+$).

Let $u_{a,v}^+$ and $x_{a,v}^+$ be the optimal solution to model \eqref{feasibility_mip}.
The feasible LOLPP solution returned by the repair step is as follows:
\begin{subequations}
\begin{align}
  \ \quad & \bar{y}_{a,v} = \hat{y}_{a,v} + u_{a,v}^+ & \quad \forall  (a,v) \in AV^{+} \\
  \ \quad & \bar{y}_{a,v} = \hat{y}_{a,v}            & \quad \forall  (a,v) \in AV \setminus AV^{+} \\
  \ \quad & \bar{x}_{a,v} = x^+_{a,v}                  & \quad \forall  (a,v) \in AV 
\end{align}
\end{subequations}

\noindent
It is important to highlight that the number of integer variables in
Model (\ref{feasibility_mip}) is $\vert AV^+ \vert$.  Effective machine learning
predictions for the number of trailers of the different types required
on each outbound terminal will result in fewer capacity violations and, hence,
fewer integer variables in Model (\ref{feasibility_mip}). The results in
Section \ref{sec: repair} compare the average
number of integer variables in the repair procedure with the number of
integer variables in Model (\ref{CreateLoadsModel}) of OLPP. One of the
key benefits of the optimization proxy is that it replaces a model
with a large number of integer decision variables with a prediction
model and a MIP model with significantly fewer variables.

\subsection{The Training Data}
\label{sec:data-augmentation}

The machine learning model is trained using a dataset of instances
$\{(\mathbf{q}_i,\mathbf{y}_i)\}_{i \in N}$, where the inputs
$\mathbf{q}_i$ are collected from historical data. The output
$\mathbf{y}_i$ corresponding to $\mathbf{q}_i$ is obtained by solving
the LOLPP. 
As a result, {\em the dataset is consistent by design}
and satisfies the field requirements: together all these instances
provide consistent load plans and commodity volume allocation
plans. It is useful to mention that, by virtue of being consistent,
this dataset is {\em easier to learn} than a dataset where the outputs
would be produced by the OLPP. Consider for instance Figure
\ref{fig:trajectories} that depicts approximations, using Piecewise
Linear Functions (PLF), of the solution trajectories of the OLPP and
the LOLPP for a given terminal and equipment. It is clear that the PLF
for the LOLPP is much easier to learn than the PLF for the OLPP.

\begin{figure}[!t]
\center
\includegraphics[width=0.65\textwidth]{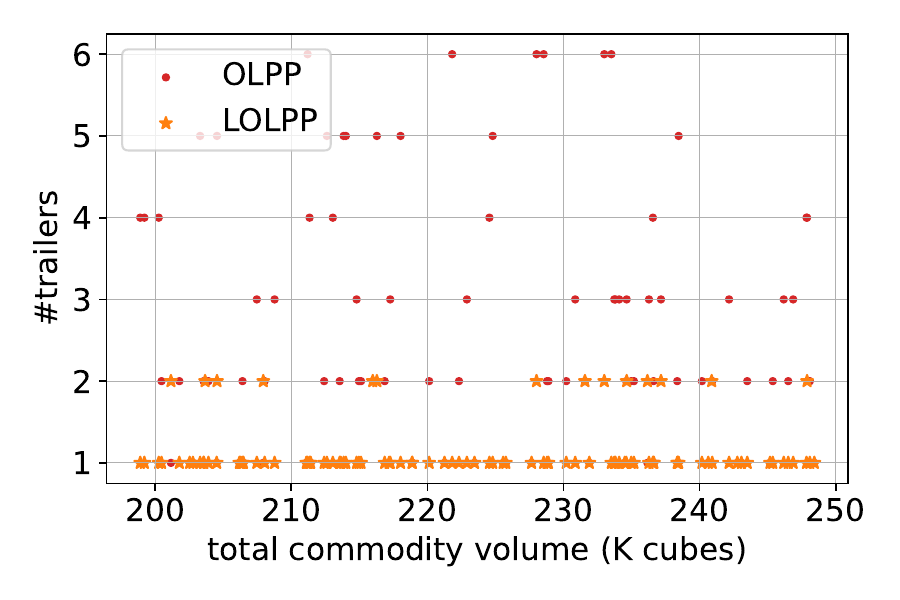}
\hfill
\caption{(color online) The Solution Trajectories of the OLPP and LOLPP on an Industrial Medium-Sized Instance.} 
\label{fig:trajectories} 
\end{figure}


The machine learning model
presented earlier is a ReLU-based neural networks that is capable of
learning piecewise linear functions. However, the complexity of
learning task depends on the shape of the PLFs. As the complexity of
the piecewise linear function grows with more pieces, a larger neural
network model is required to obtain high-quality approximations.

\begin{theorem} (Model Capacity) (\cite{arora2016understanding}) \label{thm:capacity}
Let $f: \mathbb{R}^d \rightarrow \mathbb{R}$ be a piecewise linear function with $p$ pieces. 
If $f$ is represented by a ReLU network with depth $k+1$, then it must have size at least $\frac{1}{2}kp^{\frac{1}{k}}-1$. Conversely, any piecewise linear function $f$ that is represented by a ReLU network of depth $k+1$ and size at most $s$, can have at most $(\frac{2s}{k})^k$ pieces. 
\end{theorem}

\noindent
As a consequence, given a fixed-size ReLU network, higher variability
of the solution trajectory typically results in higher approximation
errors. Therefore, the use of LOLPP helps circumvent this challenge as it generates consistent and stable solutions that are easier to learn.

%% file: givb_heuristic.tex
\section{Greedy Heuristic (GH)}
\label{section:heuristics}

To provide an additional baseline to evaluate the performance of the
optimization proxies, the paper proposes a greedy heuristic that
mimics human planners to construct feasible solutions for OLPP Model
\eqref{CreateLoadsModel}. The greedy heuristic iteratively solves
linear programs (LP) until all the $y$ variables are integers, i.e.,
they satisfy the integrality tolerance ($10^{-5}$). During each
iteration, the algorithm identifies the fractional variable with
minimum value of $(\left \lceil {y_{a,v}} \right \rceil-y_{a,v})$,
updates the lower bound of the variable $y_{a,v}$ to $\left \lceil
{y_{a,v}} \right \rceil$, and re-solves the LP as shown in Algorithm
\ref{algo: greedy heuristic}. The main idea is that, for a given outbound terminal
$a \in A$ and trailer type $v \in V_a$, if $y_{a,v}$ has a fractional
value very close to an integer $\left \lceil {y_{a,v}} \right \rceil$,
then, this indicates that there is enough commodity volume to have
\textit{at least} $\left \lceil {y_{a,v}} \right \rceil$ trailers on
the arc. The heuristic \textit{greedily} adjusts the lower bound of a
$y$ variable in each iteration until all $y$ variables can be labeled
as integers, in which case a feasible solution to OLPP has been found.

\begin{algorithm}[ht!]
\caption{Greedy Heuristic}
\label{algo: greedy heuristic}
\begin{algorithmic}[1]
\State \textit{\textbf{While}} (True )\textit{ \textbf{do}}
\State  \quad Solve LP relaxation of Model \ref{CreateLoadsModel}
\State \quad \textit{\textbf{If}} (all $y$-variables in the LP solution are integers)
\State \qquad break
\State \quad \textit{\textbf{Else}}
\State \qquad $(a^*, v^*) \leftarrow \argmin_{(a,v)}(\left\lceil {y_{a,v}} \right \rceil- y_{a,v})$ 
\State \qquad Update the lower-bound of $y_{a^*,v^*}$ to $\left \lceil {y_{a^*,v^*}} \right \rceil$
\State \quad \textit{\textbf{End If}}
\State \textit{\textbf{End While}}
\end{algorithmic}
\end{algorithm}

%% file: computational_study.tex
\section{Computational Study}
\label{section:Study}

This section reports a series of experiments conducted on industrial
instances provided by our research partner. Section \ref{instances} presents
statistics for the problem instances. Section \ref{experiment-setup}
discusses the experimental setting for the optimization models and
optimization proxies. Section \ref{compute-performance-proxies}
evaluates the computational performance of the optimization proxies
against the greedy heuristic (GH) and the optimization models (the
OLPP and and the LOLPP). Section \ref{sec: sym-break-data-gen}
evaluates the benefits of the data generation for learning. Section
\ref{sec: repair} discusses the efficiency of the repair procedure and
how it depends on the prediction accuracy.

\subsection{Instances} \label{instances}

The experiments are based on industrial instances for three different
terminals in the service network of our industry partner:
\textit{small} (S), \textit{medium} (M), and \textit{large} (L). Table
\ref{stats_instances} reports the statistics of the instances with
\textit{$\#$Terminals} denoting the total number of unique outbound
terminals, \textit{$\#$Commodities} denoting the number of commodities
that are sorted at the terminal and loaded into outbound trailers
(rounded to nearest multiple of $1,000$), and \textit{$\#$Trailers}
denoting the number of planned trailers in the reference plan for the
corresponding terminals (rounded to the nearest multiple of 50); only
pup trailers are considered in the experiments. Note that, in addition
to the planned trailers, parcel delivery companies typically operate
\textit{empty} trailers to the outbound terminals for trailer
repositioning. This study only considers trailers that are planned to
have commodity volume and does not include planned empty trailer
capacity.
\begin{table}[!t]
\centering
\caption{Data Statistics for Test Instances.}
\begin{tabular}{c|c|c|c}
\toprule
Category & \#Terminals & \#Commodities & \#Trailers\\
\hline
S & 92 & 9,000 &150 \\
M & 399 & 15,000 & 550 \\
L & 1,602 & 20,000 & 2,000 \\
\bottomrule
\end{tabular}%
\label{stats_instances}%
\end{table}%

It is worth highlighting that the L instance has more planned capacity
than the S and M instances combined. Table \ref{stats_mipmodel_var}
reports some statistics for Model \eqref{CreateLoadsModel} for the
three instances: $\#$Integer-Vars and $\#$Continuous-Vars denote the
number of integer and continuous decision variables, respectively, and
$\#$Constraints denotes the total number of constraints.

\begin{table}[!t]
\centering
\caption{Optimization Model Statistics for Model \ref{CreateLoadsModel}} 
\begin{tabular}{l|rrr}
\toprule
Instance&  \#Integer-Vars & \#Continuous-Vars & \#Constraints\\
\midrule
S & 187 & 50,713 & 8,962 \\
M & 846 & 60,023 & 15,583 \\
L & 4,475 & 107,854 & 24,931 \\
\bottomrule
\end{tabular}
\label{stats_mipmodel_var}%
\end{table}%

\subsection{The Experimental Setting}\label{experiment-setup}

\paragraph{Parameters for the LOLPP}
The cost (denoted by $d^k_a$) of assigning commodity $k$ to an outbound terminal $a \in A^k$
is defined as 
\begin{align}
d^k_a = {}& \begin{cases}
        0,& \text{if $a$ is the primary outbound terminal for commodity $k$} \\
        \left(\alpha^k_a + \beta^k \right)& \text{otherwise}
        \end{cases} \label{flow_divert}
\end{align} 
\noindent
where $\alpha^k_a$ denotes the distance between terminal $a \in A^k$
and the final destination of commodity $k \in K$. Recall that a
commodity $k \in K$ is defined as the set of all packages with the
same destination and service class. The term $\alpha^k_a$ ensures that
two commodities with different final destinations have different
flow diversion costs from the current terminal to the next outbound terminal. Parameter $\beta_k$
depends on the commodity service class and ensures that the
commodities with the same final destination (and hence, the same
value of $\alpha^k_a$) but different service classes (e.g., one-day service, two-day service, or three-days service)
have different flow diversion costs.


\paragraph{Data Generation for ML Model} 

The dataset is generated by perturbing the input parameters of
real-life instances provided by the industry partner with up to
$20,000$ commodities. Let $\mathbf{p}^{ref}$ be the volume of
different commodities in a historical plan. The OLPP/LOLPP instances
are generated by perturbing $\mathbf{p}^{ref}$. Namely, for instance
$i$, $\mathbf{p}^{(i)} = \gamma^{(i)}\times\eta^{(i)}\times
\mathbf{p}^{ref}$, where $\gamma^{(i)} \in \mathbb{R}$ denotes a
global scaling factor and $\eta \in \mathbb{R}^{|K|}$ is the commodity
level multiplicative white noise. The value $\gamma$ is sampled from a
uniform distribution $U[80\%, 120\%]$ and, for every commodity, the
value $\eta$ is sampled from a normal distribution with mean 1 and
standard deviation of $0.05$.  For every category, $10,000$ instances
are generated, and a commercial solver is used to solve the OLPP/LOLPP
models for each instance. The dataset of 10,000 instances for each
category is then split as follows: $80\%$ for the training set, $10\%$
for the validation set, and $10\%$ for the test set.

\paragraph{Performance Metrics}

The performance metrics in this study are designed to compare the
\textit{total trailer capacity} of the solutions generated by the
optimization proxies and the OLPP Model, and the \textit{consistency}
of the solutions generated by the optimization proxies and the LOLPP
model.  Given an instance $\mathbf{p}$ with optimal trailer decisions
$\mathbf{y}^*$ and a feasible trailer decision $\Bar{\mathbf{y}}$, the
optimality gap is defined as
\[\text{Gap} = (\Bar{Z} - Z^*)/|Z^*|,\]
where $Z^*$ is the optimal trailer capacity of OLPP (Model
\eqref{CreateLoadsModel}), and $\Bar{Z}$ is the trailer capacity
computed from ${\mathbf{\Bar{y}}}$. If the OLPP model cannot be solved
to optimality in 30 minutes, then the best lower bound obtained from
the solver run is used to compute the optimality gap instead of $Z^*$.


This paper proposes two metrics to quantify the consistency of a
solution. The first metric ($\Delta$)
\begin{align}
\Delta_{a,v} &= 
\begin{cases}
        |y_{a,v}  - y^{\text{ref}}_{a,v}|& \text{if $y^{\text{ref}}_{a,v} = 0$} \\
        |y_{a,v}  - y^{\text{ref}}_{a,v}|/y^{\text{ref}}_{a,v},& \text{otherwise}
        \end{cases} \ \quad \forall a \in A, v \in V_a\label{metric:normalized_distance1},\\
        \Delta &= \exp \left(\frac{1}{|A||V|} \sum_{a\in A}\sum_{v \in V_a} \log (\Delta_{a,v} + 0.01) \right) - 1\% , \label{metric:normalized_distance2}
\end{align}
where $\Delta$ is the normalized distance between the trailer
decisions in the optimized (load) plan and the trailer decisions in
the reference (load) plan $\mathbf{y}^{\text{ref}}$. It aggregates the
distances using a shifted geometric mean in order to reduce the
influence of extreme values or outliers. From a planner's perspective,
this metric captures the deviation of the trailer decisions in the
optimized plan with respect to the trailer decisions in the reference
plan. Load plans that require smaller total trailer capacity and are
as close as possible to the reference plan are highly desirable as
they reduce the planner's effort to evaluate and approve the plans.

The second metric is the total variation of the set of trailer
decisions in the $N$ instances (for each terminal). For simplicity,
instances are ordered such that $\sum_{k \in K} q^k_{i} \leq \sum_{k
  \in K} q^k_{i+1} \; \forall \ i \in N$. The goal is to analyze the
variation in trailer decisions for outbound terminals when the total
commodity volume is incrementally increased from $\sum_{k \in K}
q^k_1$ to $\sum_{k \in K} q^k_{|N|}$. Let $\{\mathbf{y}_i\}_{i \in N}$
denote the set of trailer decisions of instance set $N$. The total
variation is defined as:
\[\text{TV}(\{\mathbf{y}_i\}_{i \in N}) = \sum_{i \in N}\|\mathbf{y}_{i+1} - \mathbf{y}_{i}\|_2.\]
The total variation metric captures the sensitivity of the models,
i.e., the impact of changes in total commodity volume on the trailer
decisions for different outbound terminals. Lower total variation
implies that the trailer decisions are less sensitive to changes in
total commodity volume. Planners are more amenable to such solutions
because fewer (but effective) load plan adjustments reduce the
solution evaluation effort. They are also easier to execute in
practice.

The computational efficiency of different models is measured by the
training time of optimization proxies, including the data generation
time and the inference time. Unless specified otherwise, the average
metrics on the test dataset are reported in shifted geometric means to
reduce the effect of outliers. For a performance metric vector
$\{w\}_{i=1}^{n}$ the shifted geometric mean $\mu_s(w_1, \cdots, w_n)$
is defined as:
\[
\mu_s(w_1, \dots, w_n) = \exp \left(\frac{1}{n} \sum_{i=1}^{n} \log (w_i + s) \right) - s,
\]
where $s$ denotes the shift. The experiments use $s = 0.01$ for the
optimality gap and normalized distance metrics, $s = 1$ second for the
inference/solving time, and $s = 1$ for the distance between trailer
decisions in the optimized plan and in the reference plan.

\paragraph{Implementation Details}

All optimization models were coded in Python 3.9 using Gurobi 9.5
(\cite{gurobi}) as the MIP solver with $8$ CPU threads and default
parameter settings, except for \textit{MIPFocus} which is set to a
value of $3$.  All deep learning models are implemented using PyTorch
(\cite{paszke2019pytorch}) and trained using the Adam optimizer
(\cite{kingma2014adam}).  The ML models are multiple layer perceptron
and are hyperparameter-tuned using a grid search with learning rate in
$\{10^{-1}, 10^{-2}\}$, number of layers in $\{3, 4, 5\}$, and hidden
dimension in $\{128, 256\}$. For each system, the best model is
selected on the validation set and the performances on the test set
are reported. Experiments are conducted on dual Intel Xeon 6226@2.7GHz
machines running Linux, on the PACE Phoenix cluster (\cite{PACE}).
The training of ML models is performed on Tesla V100-PCIE GPUs with
16GBs HBM2 RAM.

\subsection{Computational Performance of the Optimization Proxies}
\label{compute-performance-proxies}

This section presents numerical experiments used to assess the
performance of the proposed optimization proxies (henceforth referred to as \textit{Proxies}) against the
optimization models and the greedy heuristic (\textit{GH}).

\paragraph{Optimality Gap}

Table~\ref{tab: opt_gap} presents the optimality gaps of various
approaches, including the results of the OLPP under various
computational time limits. In the table, the columns under OLPP report
the optimality gaps of the OLPP model with different computational
time limits. Columns \textit{Gap} and \textit{Time(s)} under
\textit{GH} and \textit{Proxies} report the optimality gaps of their
feasible solutions and the required computational times, in seconds.

Recall that the OLPP model produces
solutions that exhibit considerable variability when the total
commodity volume is perturbed; this observation is shown in
Table~\ref{tab: dist_base_plan} and \ref{tab: total_variation}. As
such, it is unlikely to be practical in scenarios with planners in the
loop. Hence, Table \ref{tab: opt_gap} compares the optimization
proxies and the greedy heuristic with what should be considered as an
``idealized'' benchmark.  With this caveat in place, observe the
performance of the optimization proxies under tight time constraints:
the proxies generate solutions with low optimality gaps and they are
up to 10 to 50 times faster than GH, and around 10 times faster than
the OLPP model solved with Gurobi.  Second, although the OLPP model
produces solutions with low optimality gaps, closing the optimality
gap proves to be a significant challenge due to the poor LP
relaxation.  The performance of GH is also impeded by the
inefficiencies of the LP relaxation, as it solves the LP relaxations
over many iterations; it takes GH around $30$ iterations for terminal
S, $200$ iterations for terminal M, and more than $1000$ iterations
for terminal L to generate a feasible solution.


\begin{table}[!t]
\centering
\caption{Optimality Gap (\%) with respect to the Total Trailer Capacity} 
\begin{tabular}{l|rrrrrrrrrr}
\toprule
\multicolumn{1}{l}{\multirow{2}{*}{Instance}} & \multicolumn{6}{c}{OLPP} & \multicolumn{2}{c}{GH} & \multicolumn{2}{c}{Proxies} \\
\cmidrule(lr){2-7} \cmidrule(lr){8-9} \cmidrule(lr){10-11}
\multicolumn{1}{l}{} & \multicolumn{1}{c}{1s} & \multicolumn{1}{c}{5s} & \multicolumn{1}{c}{10s} & \multicolumn{1}{c}{30s} & \multicolumn{1}{c}{60s} & \multicolumn{1}{c}{1800s} & \multicolumn{1}{c}{Gap} & \multicolumn{1}{c}{Time (s)} & \multicolumn{1}{c}{Gap} & \multicolumn{1}{c}{Time (s)}\\
\midrule
S & 2.59 & 0.55 & 0.48 & 0.48 & 0.48 & 0.48 & 3.84 & 3.12 & 1.14 & 0.33 \\
M & 51.15 & 5.22 & 2.18 & 1.71 & 1.41 & 1.39 & 12.85 & 13.28 & 3.80 & 1.10 \\
L & 77.35 & 14.02 & 10.41 & 2.93 & 2.07 & 0.93 & 17.01 & 121.55 & 5.21 & 2.49 \\
\bottomrule
\end{tabular}
\label{tab: opt_gap}
\end{table}



\paragraph{Consistency}

\begin{table}[!t]
\centering
\caption{Normalized Distance (\%) to Reference Load Plan} 
\begin{tabular}{l|rrrrrrrrrrrr}
\toprule
\multicolumn{1}{l}{\multirow{2}{*}{Instance}} & \multicolumn{5}{c}{OLPP ($\Delta$)} & \multicolumn{2}{c}{LOLPP} & \multicolumn{2}{c}{GH} & \multicolumn{2}{c}{Proxies} \\
\cmidrule(lr){2-6} \cmidrule(lr){7-8} \cmidrule(lr){9-10} \cmidrule(lr){11-12}
\multicolumn{1}{l}{} & \multicolumn{1}{c}{1s} & \multicolumn{1}{c}{5s} & \multicolumn{1}{c}{10s}  & \multicolumn{1}{c}{60s} & \multicolumn{1}{c}{1800s} & \multicolumn{1}{c}{$\Delta$} & \multicolumn{1}{c}{t (s)} & \multicolumn{1}{c}{$\Delta$} & \multicolumn{1}{c}{t (s)} & \multicolumn{1}{c}{$\Delta$} & \multicolumn{1}{c}{t (s)} \\
\midrule 
S & 13.53 & 13.43 & 13.42 & 13.19 & 13.20 & 5.18 &3600 & 14.39 & 3.12 & 5.19 & 0.33 \\
\midrule
M & 18.19 & 13.31 & 12.69 & 12.76 & 12.69 & 4.01  & 3600 & 17.07 & 13.28 & 3.75 & 1.10 \\
\midrule
L & 21.86 & 7.99 & 7.96 & 7.47 & 7.23 & 3.52 & 3600 & 10.08 & 121.55 & 2.97 & 2.49 \\
\bottomrule
\end{tabular}
\label{tab: dist_base_plan}
\end{table}

\begin{table}[!t]
\centering
\caption{Total Variation of Load or Trailer Decisions} 
\begin{tabular}{l|rrrrrrrrrrr}
\toprule
\multicolumn{1}{l}{\multirow{2}{*}{Instance}} & \multicolumn{5}{c}{OLPP ($TV$)} & \multicolumn{2}{c}{LOLPP} & \multicolumn{2}{c}{GH} & \multicolumn{2}{c}{Proxies} \\
\cmidrule(lr){2-6} \cmidrule(lr){7-8} \cmidrule(lr){9-10} \cmidrule(lr){11-12}
\multicolumn{1}{l}{} & \multicolumn{1}{c}{1s} & \multicolumn{1}{c}{5s} & \multicolumn{1}{c}{10s}  & \multicolumn{1}{c}{60s} & \multicolumn{1}{c}{1800s} & \multicolumn{1}{c}{$TV$} & \multicolumn{1}{c}{t (s)} & \multicolumn{1}{c}{$TV$} & \multicolumn{1}{c}{t (s)} & \multicolumn{1}{c}{$TV$} & \multicolumn{1}{c}{t (s)} \\
\midrule
S & 2097 & 2063 & 2126 & 2040 & 2054 & 392 & 3600 & 2444 & 3.12 & 115 & 0.33 \\
\hline
M & 2655 & 3275 & 2961 & 3134 & 3174 & 1201 & 3600 & 4797 & 13.28 & 347 & 1.10 \\
\hline
L & 6428 & 9174 & 8228 & 6914 & 7047 & 5228 & 3600 & 11843 & 121.55 & 565 & 2.49 \\
\bottomrule
\end{tabular}
\label{tab: total_variation}
\end{table}

Tables~\ref{tab: dist_base_plan} and \ref{tab: total_variation} report
the consistency of solutions obtained from different models in terms
of the normalized distance to the reference load plan and the total
variation of the generated solutions. When optimizing the LOLPP, the
experiments set a time limit of 1,800 seconds on the first
sub-objective. The best primal bound is then used for the optimization
of the last two sub-objectives with a time limit of 1,800 seconds
again. This second optimization optimizes a weighted function of
the second and third sub-objective where the second sub-objective has a weight of $1$ and
the third sub-objective has a small weight $\epsilon$,
i.e.,
\begin{align*}
\epsilon ={}& \frac{1}{\left(\max_{k \in K, a \in A^k} \ d^k_a \ \right)\sum_{k \in K} q^k}. 
\end{align*}
and minimizes the objective
\begin{align*}
\sum_{a \in A} \sum_{v \in V_a}\vert y_{a,v} - \gamma_{a,v} \vert + \epsilon \sum_{k \in K}\sum_{a \in A_k}\sum_{v \in V_a}d^k_a x^k_{a,v}
\end{align*}


{\em The high-level result is that proxies are ideally suited to
  produce consistent plans.} Table~\ref{tab: dist_base_plan} shows
that the proxies accurately predict, in a few seconds, the results
produced by LOLPP after an hour. Furthermore, Table \ref{tab:
  total_variation} shows that proxies produce solutions that have at
least an order of magnitude smaller total variations in trailer
decisions than both LOLPP and GH. Proxies produce load plans
that exhibit greater stability when the total commodity volume
changes for the given terminal.

The fact that proxies improve the consistency of the LOLPP plans is
especially interesting: it means that the optimization proxies, by
virtue of the learning step, avoid oscillations present in the LOLPP
approach. Of course, it does so at a small loss in objective value
(if, for instance, the LOLPP model is allowed a minute to run instead
of the 2.5 seconds of the optimizations). However, the consistency
benefits as shown in Table \ref{tab: total_variation} are substantial
and important for practical implementation. The proxies also provide
dramatic improvements over GH.



\subsection{Illustrating the Variations in Trailer Capacities}
\label{sec: sym-break-data-gen}

As discussed in Section \ref{sec:LOLPP}, the optimal (or near-optimal)
trailer decisions of the OLPP are very sensitive to changes in total
commodity volume due to the presence of symmetries in the
model. Figure \ref{fig:sym-break} compares the variations in trailer
decisions on $9$ different outbound terminals from the given terminal
in the best feasible solutions obtained from the OLPP, the LOLPP, and
the optimization proxy in the small (S) instance.

\begin{figure}[!t]
    \begin{subfigure}[t]{0.33\textwidth}
        \includegraphics[width=\textwidth]{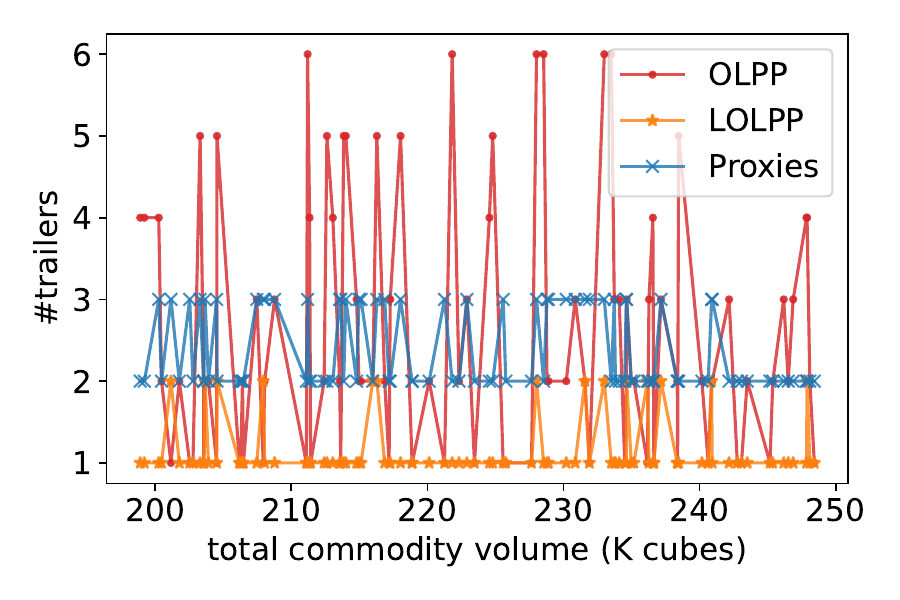}
     \end{subfigure}%
     \hfill
    \begin{subfigure}[t]{0.33\textwidth}
        \includegraphics[width=\textwidth]{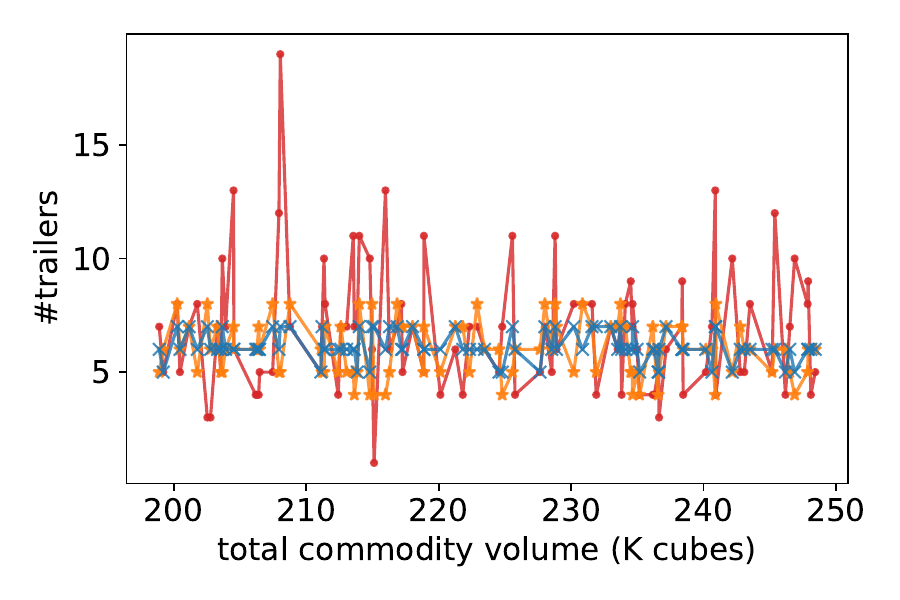}
     \end{subfigure}%
     \hfill
    \begin{subfigure}[t]{0.33\textwidth}
        \includegraphics[width=\textwidth]{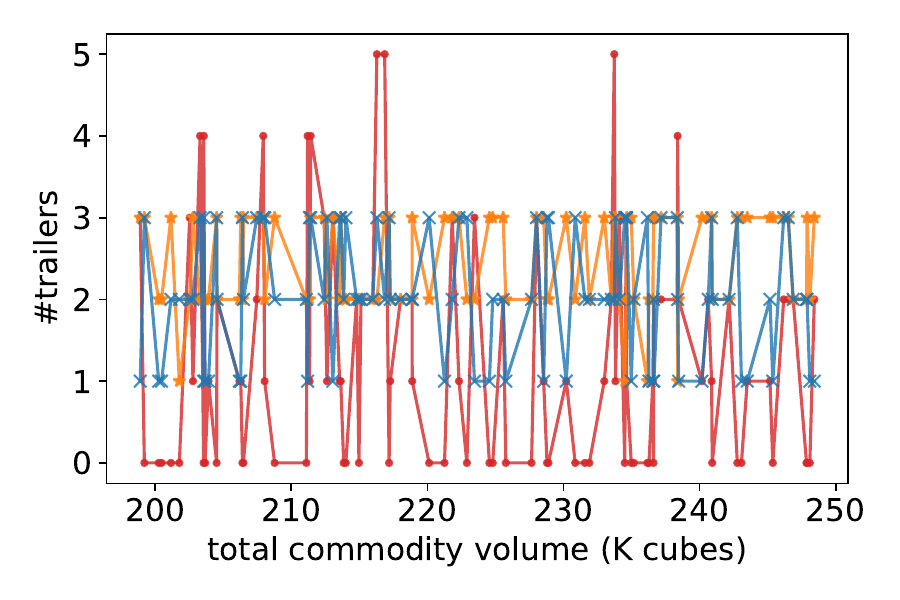}
     \end{subfigure}%
     \hfill
    \begin{subfigure}[t]{0.33\textwidth}
        \includegraphics[width=\textwidth]{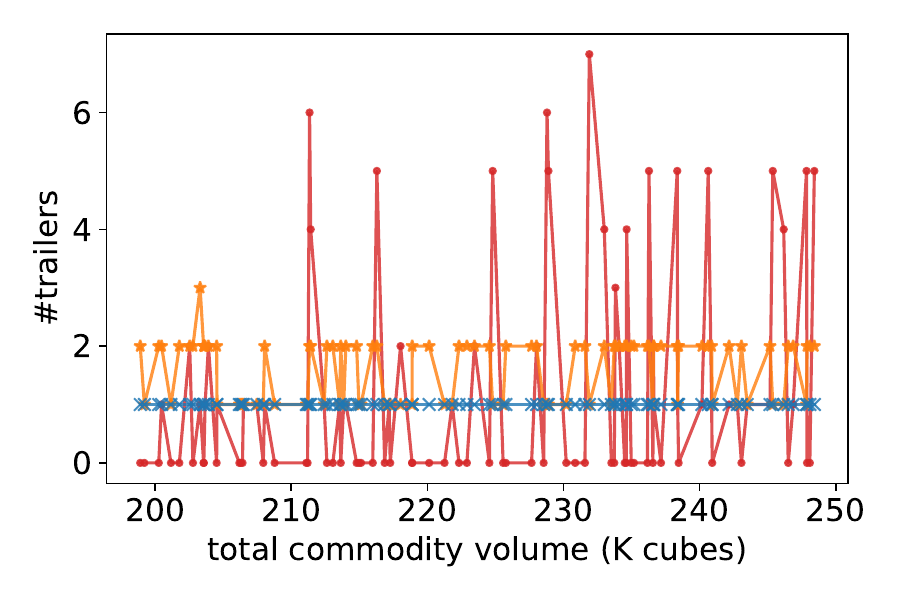}
     \end{subfigure}%
     \hfill
    \begin{subfigure}[t]{0.33\textwidth}
        \includegraphics[width=\textwidth]{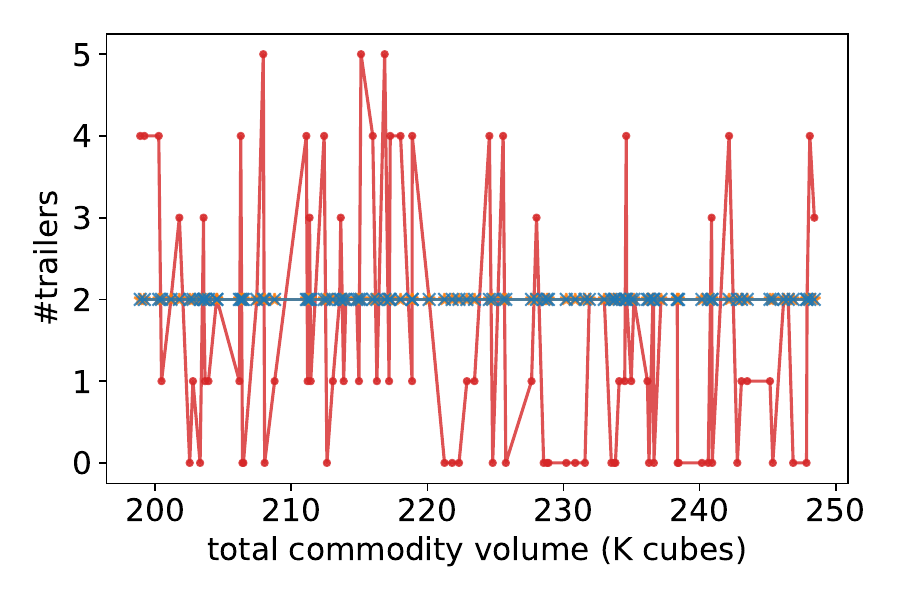}
    \end{subfigure}%
    \hfill
    \begin{subfigure}[t]{0.33\textwidth}
        \includegraphics[width=\textwidth]{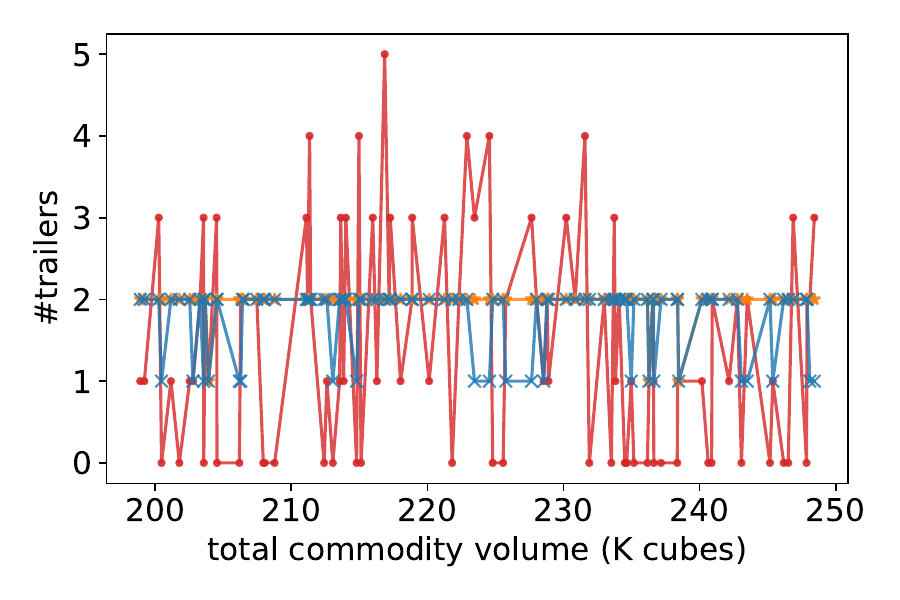}
    \end{subfigure}%
     \hfill
    \begin{subfigure}[t]{0.33\textwidth}
        \includegraphics[width=\textwidth]{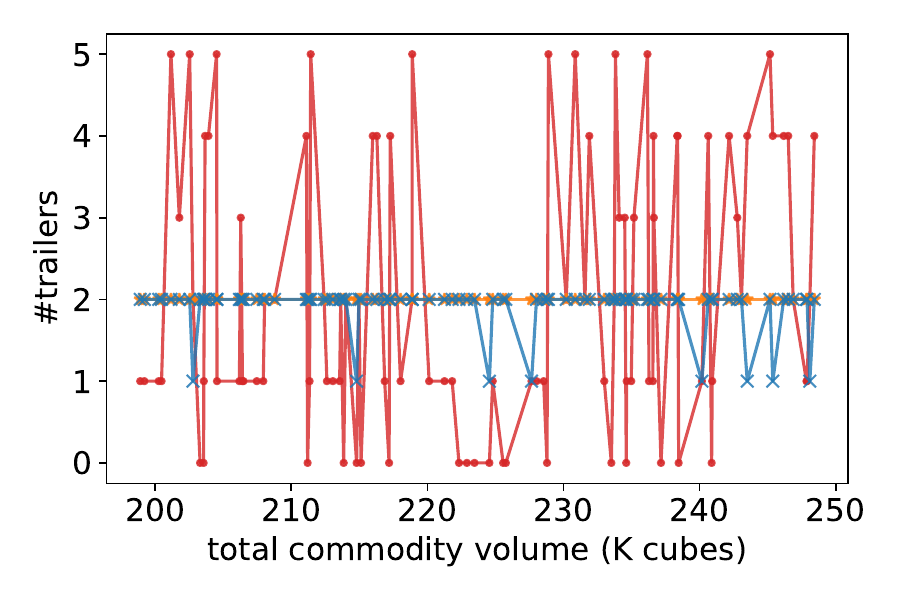}
     \end{subfigure}%
     \hfill
    \begin{subfigure}[t]{0.33\textwidth}
        \includegraphics[width=\textwidth]{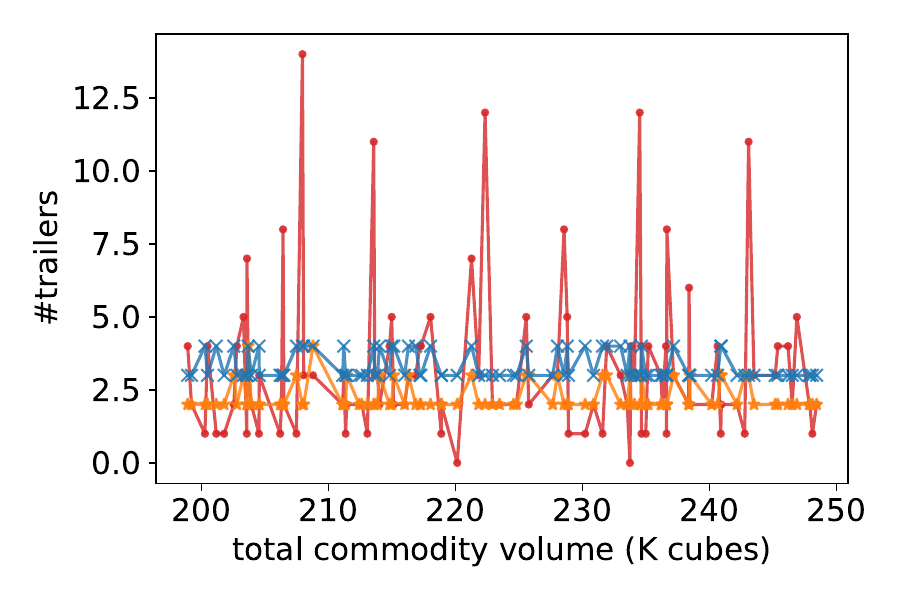}
    \end{subfigure}%
    \hfill
    \begin{subfigure}[t]{0.33\textwidth}
        \includegraphics[width=\textwidth]{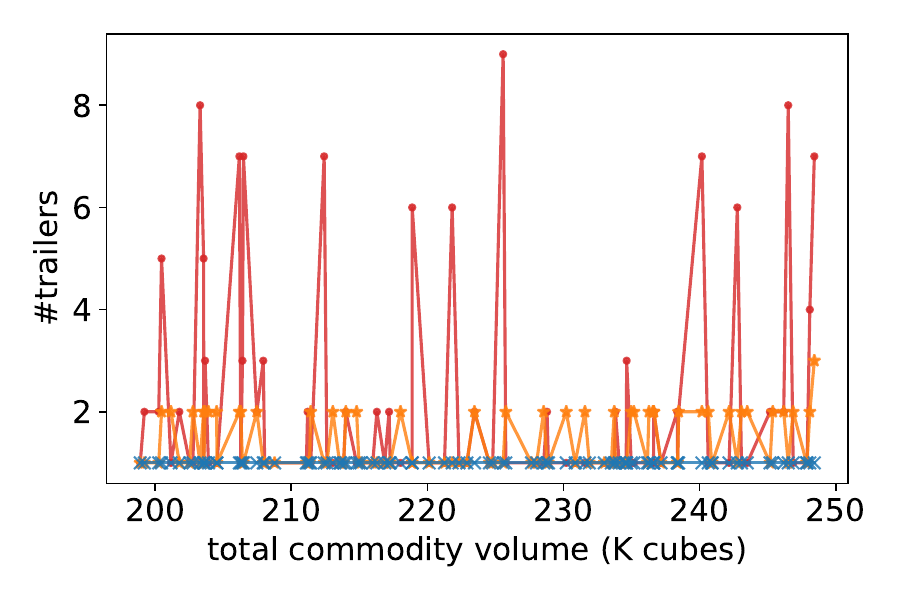}
    \end{subfigure}%
    \caption{(color online) Illustration of the sensitivity of trailer decisions to
      changes in total commodity volume on different terminals for
      small (S) instances. Trailer decisions of the OLPP are very
      sensitive to changes in total commodity volume, whereas the
      LOLPP generates consistent load plans. The proxy learns the
      LOLPP solution patterns and produces consistent
      solutions.}\label{fig:sym-break}
\end{figure}

The dramatic variation in trailer decisions on an terminal in the OLPP
solutions is not desirable in environments with planners in the loop,
where similar solutions are expected for similar instances. The LOLPP
approach is much more consistent, and its solutions are shown in
\textit{orange} in the plots of Figure \ref{fig:sym-break}. As shown
in \textit{blue} plot, the optimization proxy is effective in
producing solutions that are close to the solutions generated by
LOLPP. It should be highlighted that the LOLPP approach has two
benefits. First, it generates consistent solutions that are amenable
to planner-in-the-loop environments. Second, it makes the learning
problem much more tractable since designing a proxy for the OLPP is
really challenging due to the inherent symmetry in the solution space
and the sensitivity of the model to minor changes in commodity volume.

\subsection{Accuracy of the ML Predictions and Efficiency of the Repair Procedure}
\label{sec: repair}

Figure \ref{fig:load_proportion} shows the distribution of the
predicted trailer capacity as a percentage of the total trailer
capacity in the feasible solution generated by the proxies for each
instance. The trailer capacity predicted by the machine learning model
on each terminal is very close to a feasible solution: only a few
trailers must be added to a small set of outbound terminals to recover a
feasible solution. Indeed, the results show that more than $98\%$ of
the trailer capacity is predicted correctly and less than $2\%$ comes
from the repair procedure.

\begin{figure}[!t]  
    \centering
    \includegraphics[width=8cm, height=6cm]{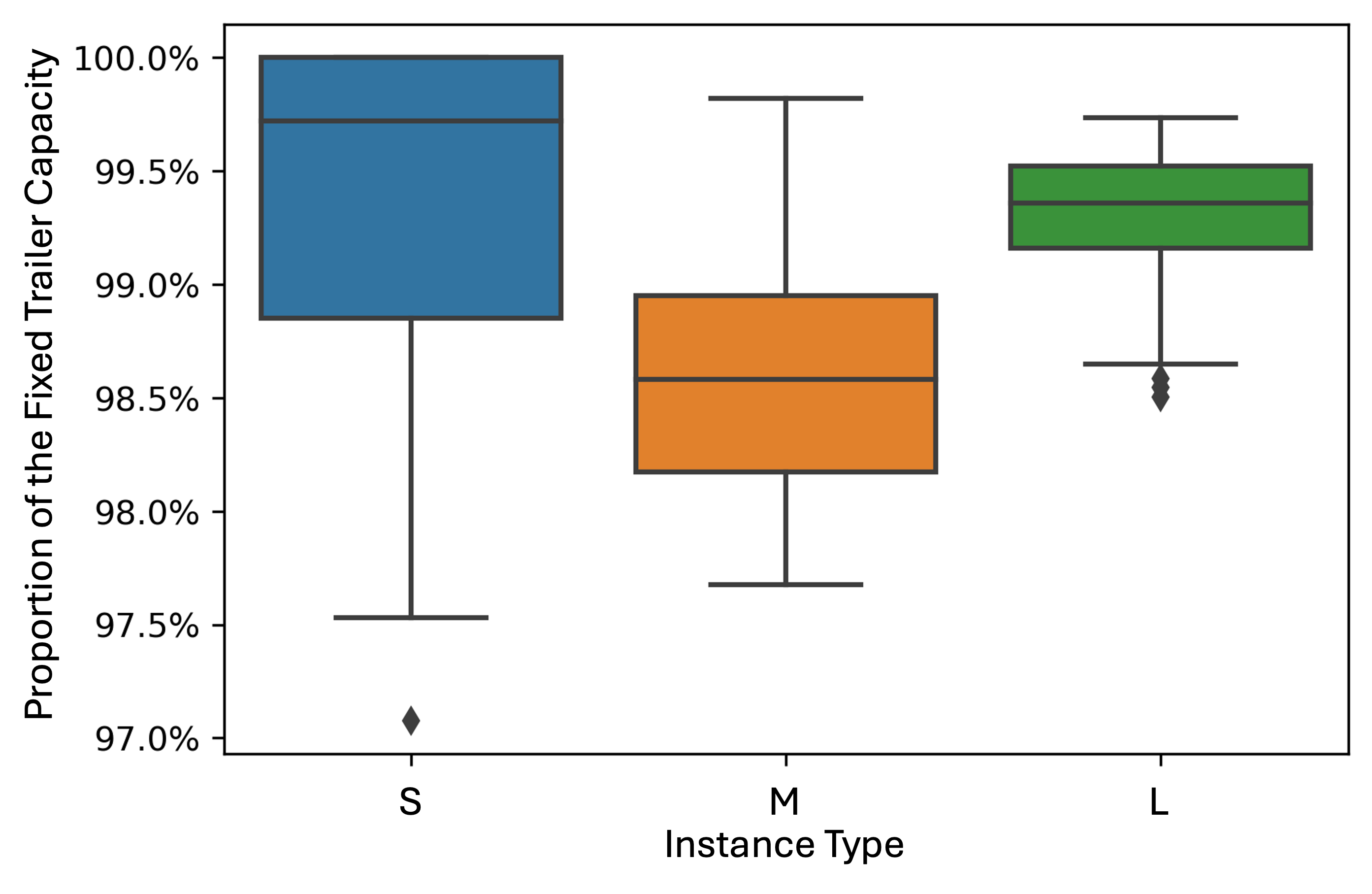}
    \caption{The Proportion of the Total Trailer Capacity in the Final Solution Fixed by the Proxies.}
    \label{fig:load_proportion}
\end{figure}

One of the key benefits of the optimization proxy is that it replaces
a model with a large number of integer decision variables with a
prediction model and a relatively simpler MIP model with a smaller
number of integer variables. Table \ref{stats_mipmodel} compares the
number of integer decision variables in the OLPP (Model
\ref{CreateLoadsModel}) and the average number of integer decision
variables in the repair procedure (Model \ref{feasibility_mip}) across
$1000$ test instances.

\begin{table}[!t]
\centering
\begin{tabular}{c|c|c}
\toprule
Instance& OLPP (Model \ref{CreateLoadsModel})  &  Repair Procedure (Model \ref{feasibility_mip})  \\
\midrule
S & 187 & 1.3  \\
M & 846 & 19.5  \\
L & 4,475 & 35.5  \\
\bottomrule
\end{tabular}
\caption{Comparing the number of integer decision variables in the OLPP, i.e., Model \ref{CreateLoadsModel}, with the average number of integer decision variables in the Repair Procedure, i.e., Model \ref{feasibility_mip}.}
\label{stats_mipmodel}%
\end{table}%

The number of integer decision variables remains the same for each
test instance in an instance category because it depends on the number
of outbound terminals and trailer types; only the commodity volume
changes across different test instances. However, the size of the MIP
Model (\ref{feasibility_mip}) in the repair procedure depends on the
predictions of the machine learning model; the number of integer
decision variables in the model is equal to the number of outbound
terminals with a capacity violation, which is determined by Model
(\ref{feasibility_recov_lp}). As the predictions can vary for
different test instances with the same set of outbound terminals due
to different commodity volumes, the number of integer variables in
Model (\ref{feasibility_mip}) can be different for different test
instances; Table \ref{stats_mipmodel} reports the average
number of integer variables in Model \ref{feasibility_mip} for the
repair procedure over all test instances. The main observation here is that the repair
procedure is efficient because it has a small number of integer
decision variables, on average, due to the accurate predictions from
the machine learning model.


\section{The Benefits of Optimization and Learning}
\label{insights-and-discussions}


The results in the paper also make it possible to evaluate the
benefits of optimization compared to human planners. In practice,
planners assign commodities to their primary terminals, then a first
alternate terminal (based on their experience), then a second
alternate terminal, and so on; planners may also prioritize alternate
terminals closest to their terminal. This is a \textit{greedy}
approach that is myopic in nature. A comparison between such a greedy
approach and the optimization models helps assess the value of
optimization. Of course, the optimization models are too slow to be
used with planners in the loop. {\em The optimization proxies proposed
  in the paper are the technology enablers for translating the
  benefits of optimization into practice.}


Figure \ref{prim_base_1alt_allalt} presents the benefits of the load
planning methodology. It compares the variation in total trailer
capacity required to contain the total commodity volume (blue curve)
and the total volume allocated to alternate terminals (green curve)
across four different load plans: \textit{Primary Only},
\textit{Reference Plan}, \textit{1-Alt} and \textit{All-Alt} for the
three instance categories. In the \textit{Primary Only} plan, each
commodity can be assigned only to its primary terminal. The
\textit{Reference Plan}, referred to as the \textit{P-Plan}, is the
reference load plan from our industry partner. The \textit{P-Plan} is
produced by the planners using the greedy approach described earlier
in this section, and each commodity in the plan can use any number of
compatible alternate terminals. In the \textit{1-Alt} plan, each
commodity can be assigned to either its primary terminal and/or the
cheapest alternate terminal; here, cheapest implies that the alternate
terminal $a$ has minimum flow diversion cost ($d^k_a$) for the
commodity $k$. In the \textit{All-Alt} plan, each commodity can be
assigned to the primary and all the available alternate
terminals. Observe that the curves in Figure
\ref{prim_base_1alt_allalt} are on different scales: the left y-axis
is for the blue curve, and the right y-axis is for the green curve.

\begin{figure}[!t]
    \begin{subfigure}[t]{0.50\textwidth}
        \includegraphics[width=\textwidth]{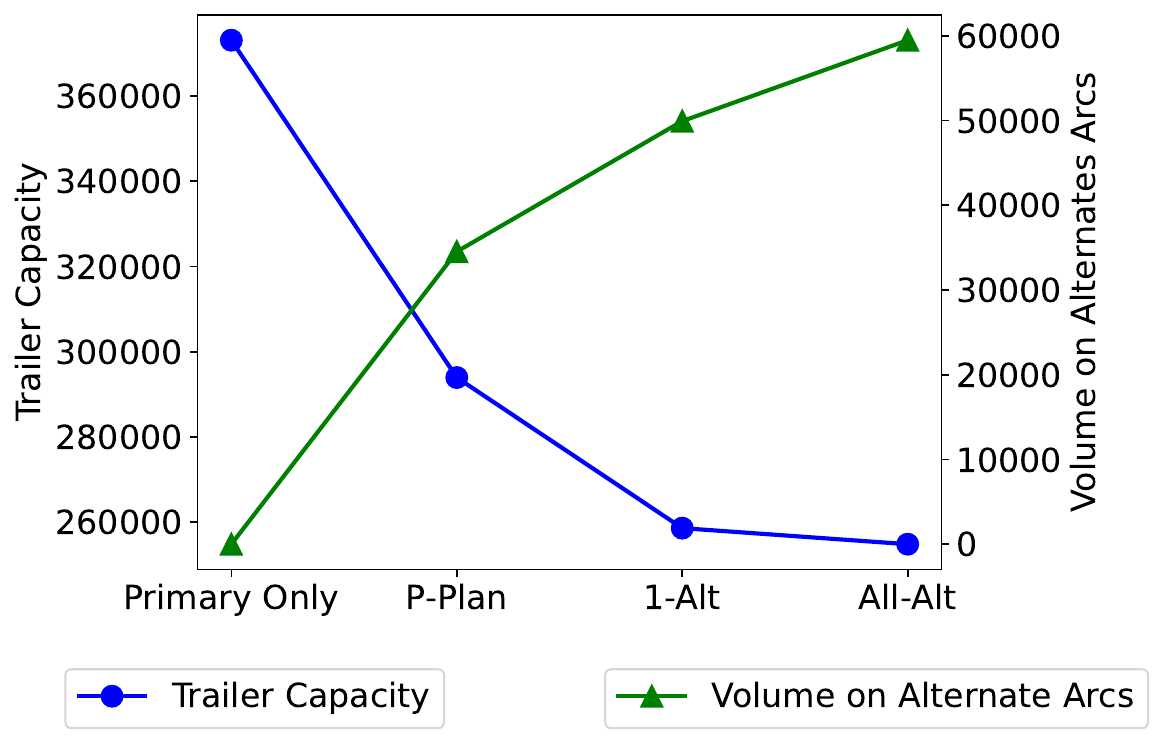}
        \caption{Instance S} \label{M1}
     \end{subfigure}%
     \hfill
    \begin{subfigure}[t]{0.50\textwidth}
        \includegraphics[width=\textwidth]{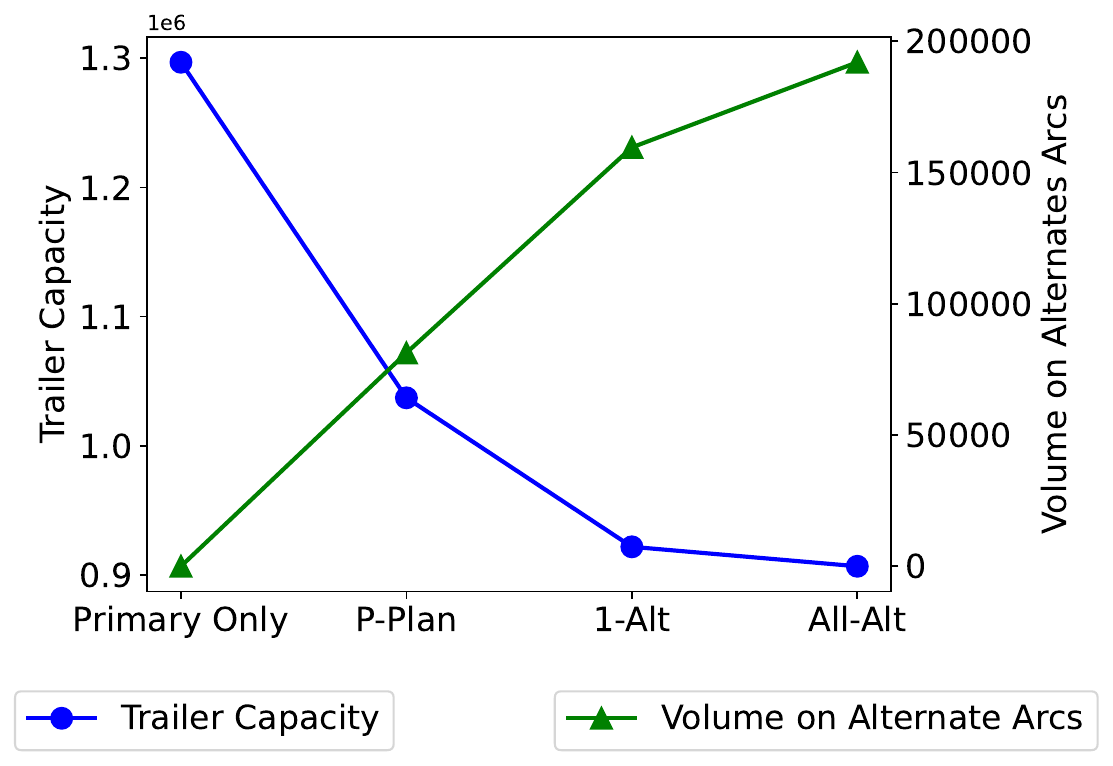}
        \caption{Instance M} \label{L2}
    \end{subfigure}%
    \hfill
    \begin{subfigure}[t]{0.48\textwidth}
        \hspace{4cm}
        \includegraphics[width=\textwidth]{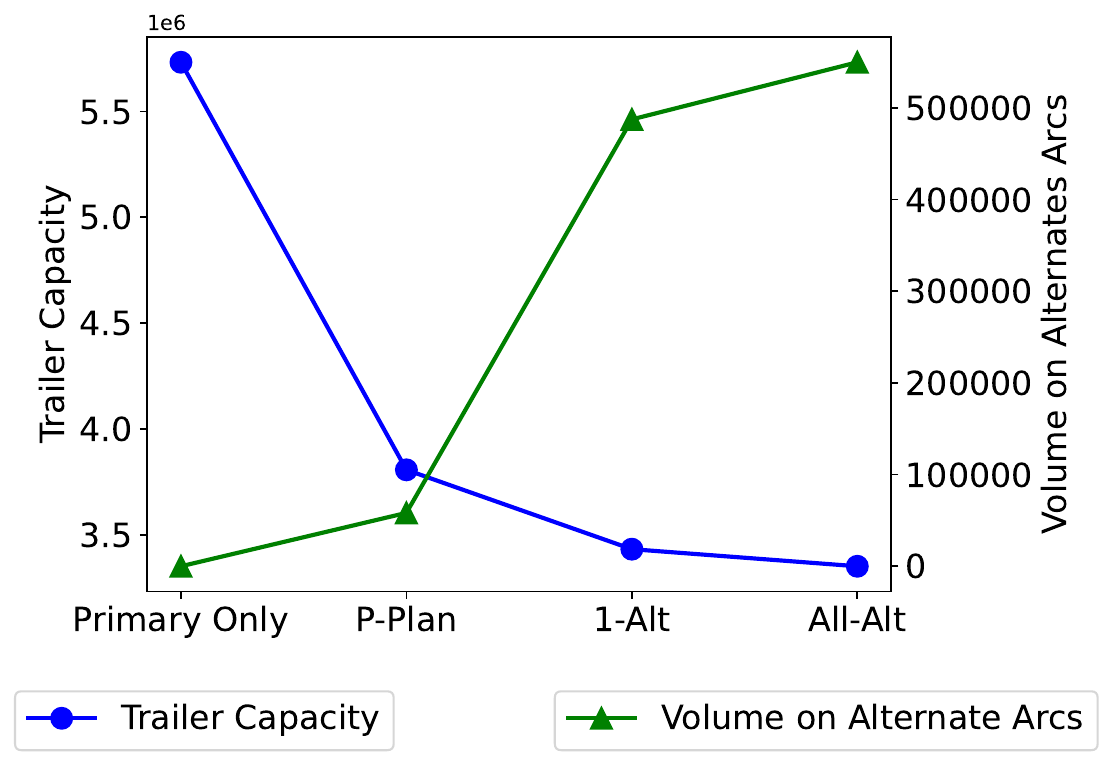}
        \caption{Instance L} \label{L}
    \end{subfigure}%
    \caption{Variation in total trailer capacity and total volume allocated to alternate terminals in four load plans.} \label{prim_base_1alt_allalt}
\end{figure}

Figure \ref{prim_base_1alt_allalt} demonstrates a consistent trend in
capacity required in the four different load plans: the capacity
monotonically decreases, and the decreases are significant. Allowing
to split the commodity volume across the primary and alternate
terminals improves trailer consolidation. These benefits are already
apparent in the \textit{P-Plan} of the planners, despite the fact that
the \textit{P-Plan} follows a greedy approach. The optimization model
corresponding to the \textit{1-Alt} plan, brings another step change
in consolidation, highlighting the benefits of optimization. This
benefit stems from the fact that a large number of commodities have at
least one alternate terminal in all instances (see Figure
\ref{Altfreq}, which highlights the percentage of commodities with a
specific number of alternate terminals for the three instances). Note
also that the \textit{1-Alt} load plan requires a significantly
smaller total trailer capacity than the \textit{P-Plan}, although the
\textit{P-Plan} has the flexibility of using any number of alternate
terminals; this is because in the \textit{P-Plan} planners use a
greedy approach to allocate commodity volume to compatible
terminals. The \textit{All-Alt} plan brings further benefits, but they
are rather incremental. Part of the reason comes from the fact that a
relatively small fraction of the commodities have more than one
alternate terminal. It would be interesting to study a network with
more flexibility as this may bring further load consolidation
benefits.

\begin{figure}[!t] 
    \centering
    \includegraphics[width=9cm, height=6cm]{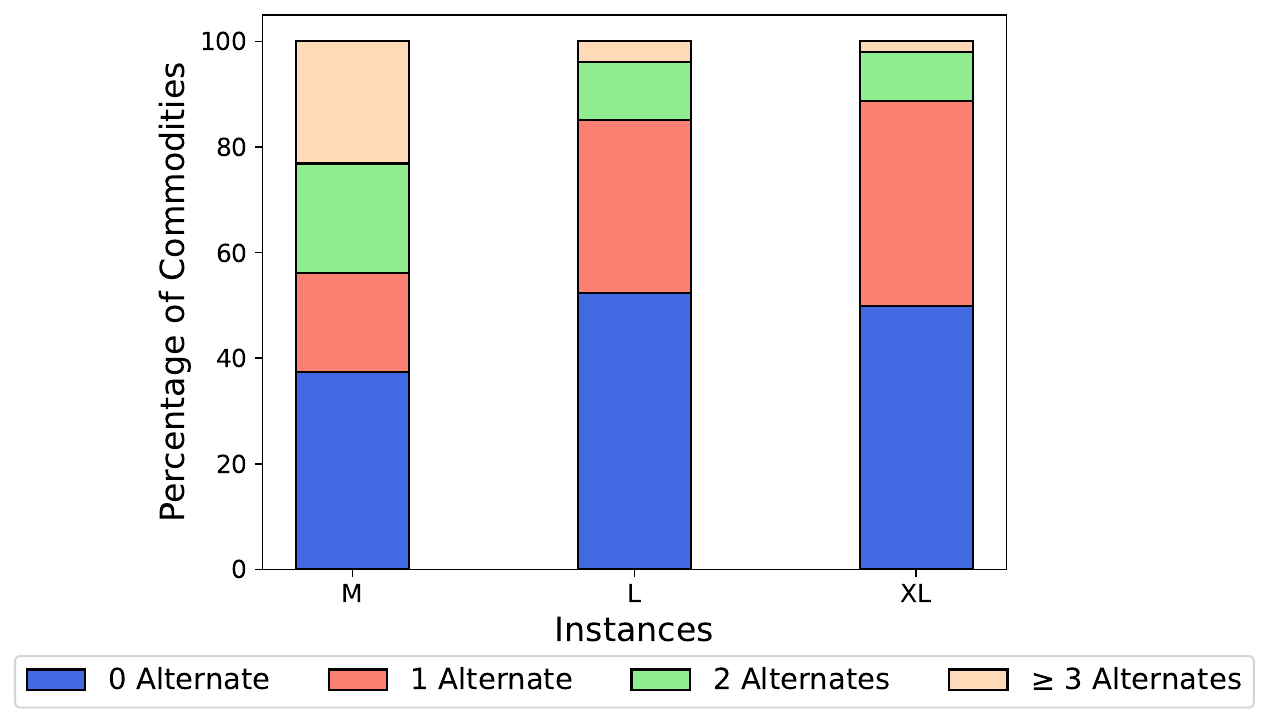}
    \caption{(color online) Instance Statistics About the Range of Alternate Terminals For Commodities.}
    \label{Altfreq}
\end{figure}

Figure \ref{vol_divert_compare} compares the percentage of the total
commodity volume that is assigned to the alternate terminals in the
\textit{P-Plan} and the \textit{All-Alt} plan. It could be undesirable to
allocate a major proportion of the volume to the alternate terminals
because the downstream buildings may not be better equipped to handle
or process the large inbound volume. Observe that, on average across
all the instances, the \textit{All-Alt} plan (resp. \textit{P-Plan})
allocates around $17\%$ (resp. $9\%$) commodity volume on the
alternate terminals. The \textit{All-Alt} plan reduces the total
trailer capacity by roughly $12\%-15\%$ relative to the
\textit{P-Plan}. For the L instance, there is a significant gap
between the \textit{P-Plan} and \textit{All-Alt} plan statistics
because most of the commodities in the \textit{P-Plan} are allocated
to the primary terminals. This is why the total commodity volume
allocated to the alternate terminals in the \textit{P-Plan} and
the \textit{Primary Only} have a small difference; see Figure
\ref{L} for the L category.

\begin{figure}[!t] 
    \centering
    \includegraphics[width=9cm, height=7cm]{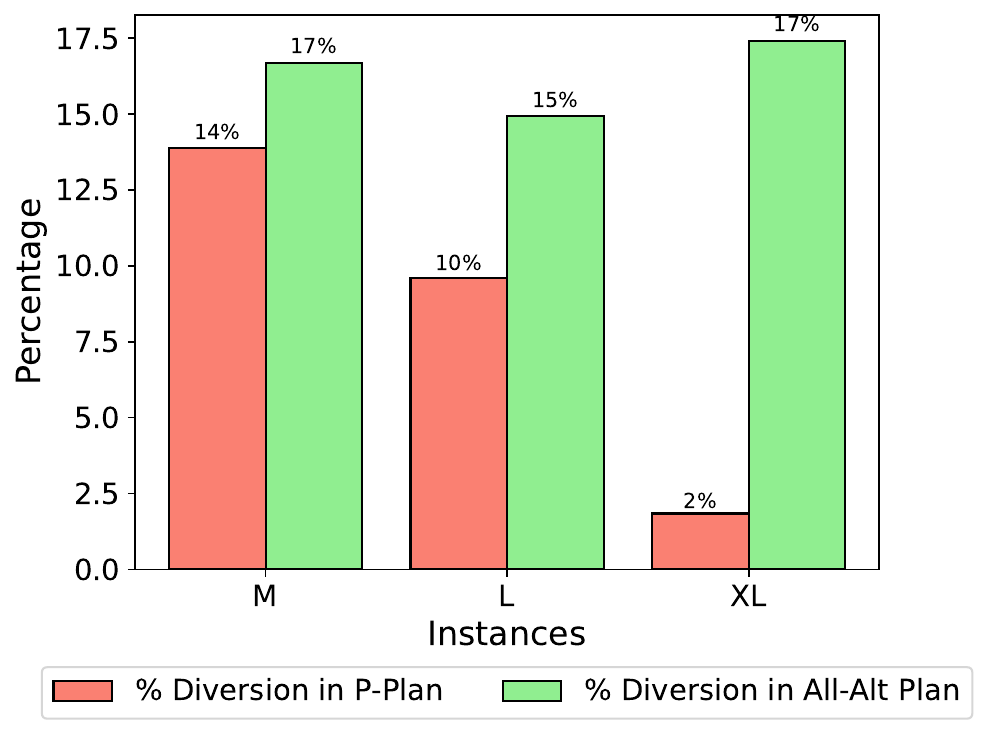}
    \caption{(color online) Proportion of total commodity volume allocated to alternate terminals in the \textit{P-Plan} and \textit{All-Alt Plan}}\label{vol_divert_compare}
\end{figure}


These results show that optimization proxies can bring substantial
benefits in practice. They provide, in real-time, significant
improvements over the existing planning process. Moreover, by virtue
of their training, they approximate the LOLPP optimization and make
sure that plans evolve smoothly during the planning process: small
changes in inputs will not result in large changes in the proposed
solutions.

{\em These results are eminently practical}. One of the challenges in
the operational setting is the need for additional trailers when the
total commodity volume increases. Based on the planner's
recommendation, the scheduling team can acquire these trailers either
through empty trailer repositioning or by engaging in short-term
trailer leasing with local companies.  Conversely, if the commodity
volume decreases, planners are left with a plan that has low trailer
utilization. The optimization proxies address this issue
directly. Planners can also use the proposed optimization proxies to
obtain recommendations for load plan adjustment within a matter of
seconds in the event of a disruption (due to uncertainty in commodity
volume), even for the largest terminal. Furthermore, the
recommendations from the optimization proxies are consistent with
existing load plans, which makes it easy for the planners to evaluate
and implement the suggestions. Finally, new terminals in the service
network often do not have dedicated planners to develop and/or update
existing load plans when commodity volume changes, and extra capacity
is built in the system to handle the commodity volume in the
worst-case scenario. Optimization proxies can be used as a
decision-support tool at such terminals.


%% file: appendix.tex
\subsection{Example for Compatible Primary and Alternate Terminals for Commodities from a Terminal with Multiple Sorts}\label{load_pair_idea_arcs} 

In this section, we present an example to illustrate the construction of set $A^k$ for a commodity $k \in K$ when OLPP considers multiple sorts at a terminal on a specific day. We define $A^k$ to denote the set of primary and 
 alternate arcs $(o_a, \hat{d}_a)$ such that commodity $k$ is processed at terminal-sort pair $o_a$ and loaded into a trailer planned for outbound terminal-sort pair $\hat{d}_a$. Figure \ref{fig: load_pair_new1} shows an example where there are two sorts at terminal A on day $1$: sort $1$ and sort $2$. Suppose the nodes of terminals B, D, and F denote a specific sort $\hat{s}$ at the respective terminals. One unit of a commodity destined for terminal D on day $4$ is processed at terminal A in sort $2$; the commodity has a primary arc to terminal F on day $3$ and an alternate arc to terminal B on day $3$. A commodity with four units is processed at terminal A in sort $1$ and has a primary arc to its ultimate destination terminal B on day $3$. Note that if we restrict the volume of a commodity to be assigned only to its primary and alternate arcs, two trailers are required. However, in practice, as shown in Figure \ref{fig: load_pair_new2}, a trailer is opened (for loading) at terminal A in sort $1$, and the four units of commodity B are loaded into the trailer. As the trailer has available space, it is kept open till sort $2$, where one unit of commodity D is loaded into the trailer. Finally, the trailer transports the commodities to terminal B on day $3$; as a result, only one trailer is required because both the commodities are consolidated into the same trailer by keeping the trailer waiting from sort $1$ to sort $2$. It is important to note that consolidation is possible because the departures from sort $1$ and sort $2$ from terminal A are planned to arrive at the same destination node at terminal B. To allow such consolidations in the OLPP model, we define $A^k$ for commodity $k = D$ to include the arc $(A, sort 1, Day 1) \rightarrow(B, \hat{s}, Day 3)$. More formally, for any commodity $k \in K$ with a primary and alternate arc set $A^k$ we include arcs $a' = (o_{a'}, \hat{d}_{a'})$ in $A^k$ if and only if there exists an arc $(o_a, \hat{d}_a) \in A^k$ such that the sort of node $o_{a'}$ is earlier (in time) than sort of node $o_{a}$ and $\hat{d}_a = \hat{d}_{a'}$. 
\begin{figure}[!ht]
    \begin{subfigure}[t]{0.48\textwidth}
        \includegraphics[width=\textwidth]{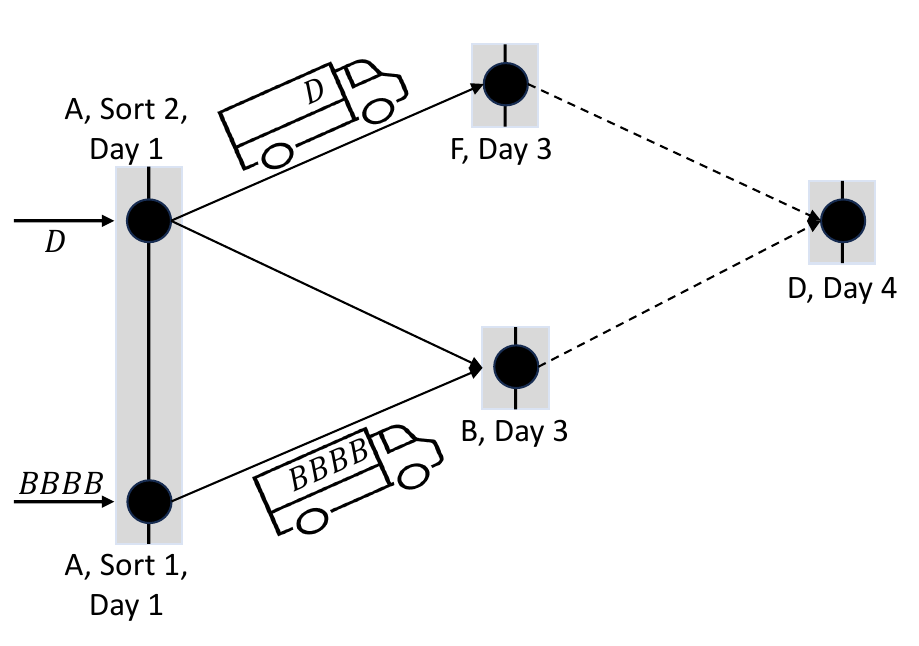}
        \caption{base case}\label{fig: load_pair_new1}
     \end{subfigure}%
     \hfill
    \begin{subfigure}[t]{0.48\textwidth}
        \includegraphics[width=\textwidth]{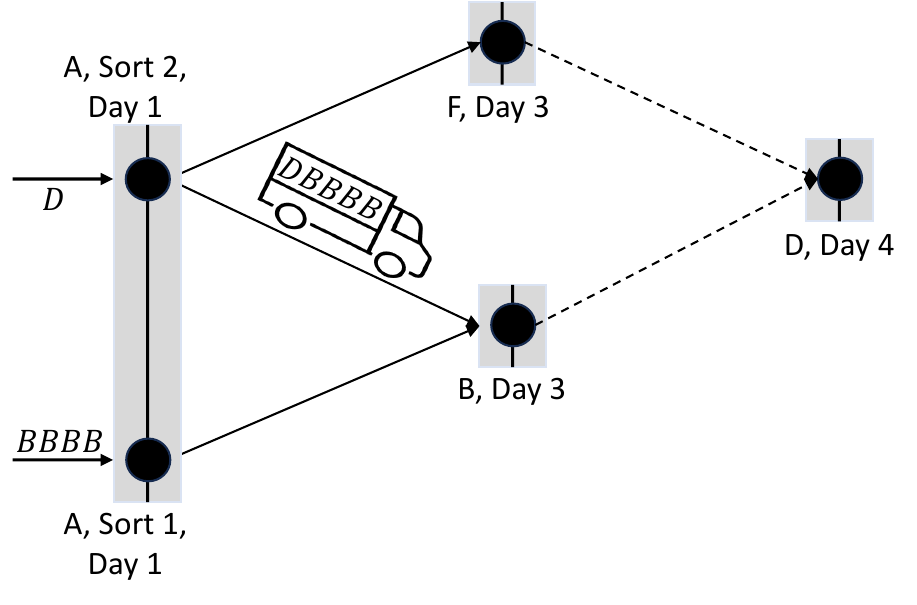}
        \caption{better consolidation}\label{fig: load_pair_new2}
     \end{subfigure}%
     \hfill
    \caption{Example figure to illustrate the additional consolidation opportunities arising from delaying a trailer from one sort to another sort at the same terminal}\label{fig: load pair idea}
\end{figure}

\subsection{Complexity Results}
\label{appendix:complexity}

Model \ref{CreateLoadsModel} is difficult to solve because in addition to determining the right combination of trailer types to contain volume assigned to each outbound terminal, we need to determine the right splits of commodity volume for the given set of outbound terminals. We will analyze the complexity of Model \ref{CreateLoadsModel} using the special cases described below.

\textit{\textbf{Case 1:}} There is only one trailer type available at the terminal, i.e., $\vert V_a \vert = 1 \ \forall a \in A$. Each commodity $k \in K$ is compatible with exactly one outbound terminal $a_k$, i.e., $A^k = \{a_k\}\ \forall k \in K$

\textit{\textbf{Case 2:}} There is only one trailer type available at the terminal, i.e., $\vert V_a \vert = 1 \ \forall a \in A$. Each commodity $k \in K$ is compatible with all outbound terminals, i.e., $A^k = A\ \forall k \in K$

\begin{proposition}
    Cases $1$ and $2$ are polynomial time solvable
\end{proposition}
\begin{proof}
In \textit{Case} $1$, the volume of each commodity $k$ is assigned to its only outbound terminal, $a_k$, i.e. $x^k_{a_k} = q^k$. Then, the optimal solution has $y_a = \left \lceil{\frac{\sum_{k \in K: a \in A^k} x^k_{a}}{Q}}\right \rceil = \left \lceil{\frac{\sum_{k \in K: a \in A^k} q^k}{Q}}\right \rceil \ \forall \ a \in A$.

In \textit{Case} $2$, the optimal solution is to assign the volume of all commodities on any arc $a \in A$ and set $x^k_a = q^k \ \forall k \in K$, $y_a = \left \lceil{\frac{\sum_{k \in K} q^k}{Q}}\right \rceil, y_{a'} = 0 \ \forall \ a' \in A, a' \neq a$.
\end{proof}

\textit{\textbf{Case 3:}} Same as \textit{Case 1}, but with more than one trailer type available at the terminal

\textit{\textbf{Case 4:}} Same as \textit{Case 2}, but with more than one trailer type available at the terminal

\begin{proposition}
    Cases $3$ and $4$ are weakly NP-Hard
\end{proposition}
\begin{proof}
In the optimal solution in \textit{Case} $3$, the volume of each commodity $k$ is assigned to its only outbound terminal $a_k$. Thus, it remains to decide the optimal combination of trailer types required to containerize the volume assigned to every outbound terminal. This is the minimum knapsack problem (see \cite{csirik1991heuristics} for the problem definition) for each outbound terminal (that has more than one trailer type) as shown in \ref{mkp}, which is known to be weakly NP-Hard.
\begin{subequations}\label{mkp}
    \begin{align}
    \text{For every $a \in A$: }\underset{y}{\text{Minimize}} \ \quad &   \sum_{v \in V_a} c_{v} y_{a,v} \label{mkp_obj}\\
    \text{subject to} 
         \ \quad & \sum_{k \in K: a \in A^k}  q^k\leq \sum_{v \in V_a}Q_v (y_{a,v}), & \quad  \label{mkp_LoadCapacity}\\
         \ \quad & y_{a,v} \in \mathbb{Z}_{\geq 0} & \quad \forall v \in V_a  \label{mkp_Nonneg2}
    \end{align}
\end{subequations}

Similarly, for \textit{Case} $4$, there exists an optimal solution in which the volume of all commodities is assigned to one outbound terminal $a^* \in A$, i.e., $x^k_{a^*} = q^k \ \forall k \in K$ and it remains to solve a minimum knapsack problem for the outbound terminal $a^*$ due to which \textit{Case} 4 is weakly NP-Hard.
\end{proof}

\textit{\textbf{Case 5:}} Each commodity $k \in K$ is compatible with a subset of outbound terminals, i.e., $A^k \subset A$, and has unit volume, $q^k = 1$. There is only one trailer type with per-unit cost $c_{a}=1 \ \forall a \in A$ and capacity $Q=\max_{a \in A} \{\sum_{k \in K} \mathbb{1}_{a \in A^k} \}$; hence, $y_a \in \{0,1\} \forall a \in A$, as installing one unit of trailer is sufficient to contain the total volume that can be assigned to the outbound terminal. Note that we ignore the index $v$ for the trailer because each outbound terminal has exactly one and the same trailer type.

\begin{claim}\label{claim_case5}
In the optimal solution of \textit{Case} 5, each commodity is assigned to exactly one outbound terminal (i.e. there is no splitting of volume)
\end{claim}
\begin{proof}
    We will present a proof by contradiction. WLOG, suppose there exists an optimal solution in which the volume of a commodity $\hat{k}$ is split between two outbound terminals and the volume of all other commodities $k \in K \backslash \{ \hat{k} \}$  is assigned to exactly one outbound terminal $a_k$. Thus, we have $x^k_{a_k} = q^k \ \forall k \in K \backslash \{ \hat{k} \}$ and $x^{\hat{k}}_{a_1} + x^{\hat{k}}_{a_2} = q^k$. Consider a solution $\underbar{x}^k_a =x^k_a \ \forall k \in K \backslash \{\hat{k}\}$ and $\underbar{x}^{\hat{k}}_{a_1} = x^{\hat{k}}_{a_1} + \epsilon,\underbar{x}^{\hat{k}}_{a_2} = x^{\hat{k}}_{a_2} - \epsilon$, where $\epsilon > 0$ is a small real number. Note that $\underbar{x}^{\hat{k}}_{a_1} + \underbar{x}^{\hat{k}}_{a_2} = q$. Consider another solution $\bar{x}^k_a = x^k_a \ \forall k \in K \backslash \{\hat{k}\}$ and $\bar{x}^{ \hat{k}}_{a_1} = x^{\hat{k}}_{a_1} -\epsilon,\bar{x}^{\hat{k}}_{a_2} = x^{\hat{k}}_{a_2} + \epsilon$. Note that both solutions $\underbar{x}$ and $\bar{x}$ satisfy constraints \eqref{FlowAssignment} and are feasible to constraints \eqref{LoadCapacity} because we choose $Q = \max_{a \in A} \{\sum_{k \in K} \mathbb{1}_{a \in A^k} \}$. The solution $x$ can be written as a convex combination of the solution $\underbar{x}$ and $\bar{x}$ ($x^k_a = \frac{1}{2}\bar{x}^{k}_{a} + \frac{1}{2}\underbar{x}^{k}_{a} \ \forall k \in K, a \in A^k$) which contradicts the optimality of the solution. 
\end{proof}
\begin{proposition}
    Case 5 is strongly NP-Hard
\end{proposition}
\begin{proof}
We will show that this special case can be solved as a set cover problem, which is known to be strongly NP-Hard (\cite{karp2010reducibility}). An instance of a set cover is given by a ground set $U = \{x_1, x_2, \cdots, x_n\}$ and a collection of $m$ subsets $E_i \subseteq U \ \forall i \in \{1,2,\cdots,m\}$ of the ground set $U$. The optimization problem is to find the smallest number of subsets $i \in \{1,2,\cdots,m\}$ such that $\bigcup_{i \in \{1,2,\cdots,m\}} E_i = U$.

From claim \ref{claim_case5}, we know that each commodity is assigned to exactly one outbound terminal in the optimal solution. Let commodity $k \in K$ denote element $x_k \in U$, $\vert K \vert = n$ and set of outbound terminals $A = \{1,2,\cdots,m\}$. Define $K_i = \{k \in K: x_k \in E_i\}$ as the set of commodities or elements that can be covered by selecting outbound terminals $i \in \{1,2,\cdots,m\}$. Now note that finding the smallest number of outbound terminals $a \in A$ such that all commodities in $K$ are covered is equivalent to finding the smallest number of subsets $i \in \{1,2,\cdots,m\}$ to cover all elements in $U$. 
\end{proof}

%% file: references.bib
@article{boland2017continuous,
  title={The continuous-time service network design problem},
  author={Boland, Natashia and Hewitt, Mike and Marshall, Luke and Savelsbergh, Martin},
  journal={Operations research},
  volume={65},
  number={5},
  pages={1303--1321},
  year={2017},
  publisher={INFORMS}
}

@article{erera2013improved,
  title={Improved load plan design through integer programming based local search},
  author={Erera, Alan and Hewitt, Michael and Savelsbergh, Martin and Zhang, Yang},
  journal={Transportation Science},
  volume={47},
  number={3},
  pages={412--427},
  year={2013},
  publisher={INFORMS}
}

@article{lindsey2016improved,
  title={Improved integer programming-based neighborhood search for less-than-truckload load plan design},
  author={Lindsey, Kathleen and Erera, Alan and Savelsbergh, Martin},
  journal={Transportation science},
  volume={50},
  number={4},
  pages={1360--1379},
  year={2016},
  publisher={INFORMS}
}

@incollection{bakir2021motor,
  title={Motor Carrier Service Network Design},
  author={Bakir, Ilke and Erera, Alan and Savelsbergh, Martin},
  booktitle={Network Design with Applications to Transportation and Logistics},
  pages={427--467},
  year={2021},
  publisher={Springer}
}

@article{marshall2021interval,
  title={Interval-based dynamic discretization discovery for solving the continuous-time service network design problem},
  author={Marshall, Luke and Boland, Natashia and Savelsbergh, Martin and Hewitt, Mike},
  journal={Transportation science},
  volume={55},
  number={1},
  pages={29--51},
  year={2021},
  publisher={INFORMS}
}

@online{WinNT,
  author = {Morgan Stanley},
  year = 2022,
  url = {https://www.morganstanley.com/ideas/global-ecommerce-growth-forecast-2022}
}

@article{hewitt2019enhanced,
  title={Enhanced dynamic discretization discovery for the continuous time load plan design problem},
  author={Hewitt, Mike},
  journal={Transportation science},
  volume={53},
  number={6},
  pages={1731--1750},
  year={2019},
  publisher={INFORMS}
}

@article{herszterg2022near,
  title={Near real-time loadplan adjustments for less-than-truckload carriers},
  author={Herszterg, Ian and Ridouane, Yassine and Boland, Natashia and Erera, Alan and Savelsbergh, Martin},
  journal={European Journal of Operational Research},
  volume={301},
  number={3},
  pages={1021--1034},
  year={2022},
  publisher={Elsevier}
}

@article{csirik1991heuristics,
  title={Heuristics for the 0-1 min-knapsack problem},
  author={Csirik, J{\'a}nos},
  journal={Acta Cybernetica},
  volume={10},
  number={1-2},
  pages={15--20},
  year={1991},
  publisher={University of Szeged}
}

@book{karp2010reducibility,
  title={Reducibility among combinatorial problems},
  author={Karp, Richard M},
  year={2010},
  publisher={Springer}
}

@article{baubaid2021value,
  title={The value of limited flexibility in service network designs},
  author={Baubaid, Ahmad and Boland, Natashia and Savelsbergh, Martin},
  journal={Transportation Science},
  volume={55},
  number={1},
  pages={52--74},
  year={2021},
  publisher={INFORMS}
}

@article{friesen1986variable,
  title={Variable sized bin packing},
  author={Friesen, Donald K. and Langston, Michael A.},
  journal={SIAM journal on computing},
  volume={15},
  number={1},
  pages={222--230},
  year={1986},
  publisher={SIAM}
}

@article{crainic1997planning,
  title={Planning models for freight transportation},
  author={Crainic, Teodor Gabriel and Laporte, Gilbert},
  journal={European journal of operational research},
  volume={97},
  number={3},
  pages={409--438},
  year={1997},
  publisher={Elsevier}
}

@article{powell1986local,
  title={A local improvement heuristic for the design of less-than-truckload motor carrier networks},
  author={Powell, Warren B},
  journal={Transportation Science},
  volume={20},
  number={4},
  pages={246--257},
  year={1986},
  publisher={INFORMS}
}

@phdthesis{ulch2022greedy,
  title={Greedy Approaches for Large-Scale Flow and Load Planning},
  author={Ulch, Daniel},
  year={2022},
  school={Georgia Institute of Technology}
}

@article{baubaid2023dynamic,
  title={The Dynamic Freight Routing Problem for Less-Than-Truckload Carriers},
  author={Baubaid, Ahmad and Boland, Natashia and Savelsbergh, Martin},
  journal={Transportation Science},
  volume={57},
  number={3},
  pages={717--740},
  year={2023},
  publisher={INFORMS}
}

@article{kingma2014adam,
  title={Adam: A method for stochastic optimization},
  author={Kingma, Diederik P and Ba, Jimmy},
  journal={arXiv preprint arXiv:1412.6980},
  year={2014}
}

@misc{gurobi,
  author = {{Gurobi Optimization, LLC}},
  title = {{Gurobi Optimizer Reference Manual}},
  year = 2023,
  url = "https://www.gurobi.com"
}

@article{paszke2019pytorch,
  title={Pytorch: An imperative style, high-performance deep learning library},
  author={Paszke, Adam and Gross, Sam and Massa, Francisco and Lerer, Adam and Bradbury, James and Chanan, Gregory and Killeen, Trevor and Lin, Zeming and Gimelshein, Natalia and Antiga, Luca and others},
  journal={Advances in neural information processing systems},
  volume={32},
  year={2019}
}

@manual{PACE,
     title  = "{P}artnership for an {A}dvanced {C}omputing {E}nvironment ({PACE})",
     author = "PACE",
     url    = "http://www.pace.gatech.edu",
     year   = "2017"
}

@article{arora2016understanding,
  title={Understanding deep neural networks with rectified linear units},
  author={Arora, Raman and Basu, Amitabh and Mianjy, Poorya and Mukherjee, Anirbit},
  journal={arXiv preprint arXiv:1611.01491},
  year={2016}
}

@article{kotary2021learning,
  title={Learning hard optimization problems: A data generation perspective},
  author={Kotary, James and Fioretto, Ferdinando and Van Hentenryck, Pascal},
  journal={Advances in Neural Information Processing Systems},
  volume={34},
  pages={24981--24992},
  year={2021}
}

@article{bengio2021machine,
  title={Machine learning for combinatorial optimization: a methodological tour d’horizon},
  author={Bengio, Yoshua and Lodi, Andrea and Prouvost, Antoine},
  journal={European Journal of Operational Research},
  volume={290},
  number={2},
  pages={405--421},
  year={2021},
  publisher={Elsevier}
}

@article{kotary2021end,
  title={End-to-end constrained optimization learning: A survey},
  author={Kotary, James and Fioretto, Ferdinando and Van Hentenryck, Pascal and Wilder, Bryan},
  journal={arXiv preprint arXiv:2103.16378},
  year={2021}
}

@inproceedings{kotary2022fast,
  title={Fast approximations for job shop scheduling: A lagrangian dual deep learning method},
  author={Kotary, James and Fioretto, Ferdinando and Van Hentenryck, Pascal},
  booktitle={Proceedings of the AAAI Conference on Artificial Intelligence},
  volume={36},
  number={7},
  pages={7239--7246},
  year={2022}
}

@article{park2022confidence,
  title={Confidence-Aware Graph Neural Networks for Learning Reliability Assessment Commitments},
  author={Park, Seonho and Chen, Wenbo and Han, Dahye and Tanneau, Mathieu and Van Hentenryck, Pascal},
  journal={arXiv preprint arXiv:2211.15755},
  year={2022}
}

@article{yuan2022reinforcement,
  title={Reinforcement learning from optimization proxy for ride-hailing vehicle relocation},
  author={Yuan, Enpeng and Chen, Wenbo and Van Hentenryck, Pascal},
  journal={Journal of Artificial Intelligence Research},
  volume={75},
  pages={985--1002},
  year={2022}
}

@article{khalil2017learning,
  title={Learning combinatorial optimization algorithms over graphs},
  author={Khalil, Elias and Dai, Hanjun and Zhang, Yuyu and Dilkina, Bistra and Song, Le},
  journal={Advances in neural information processing systems},
  volume={30},
  year={2017}
}

@article{kool2018attention,
  title={Attention, learn to solve routing problems!},
  author={Kool, Wouter and Van Hoof, Herke and Welling, Max},
  journal={arXiv preprint arXiv:1803.08475},
  year={2018}
}

@inproceedings{fu2021generalize,
  title={Generalize a small pre-trained model to arbitrarily large tsp instances},
  author={Fu, Zhang-Hua and Qiu, Kai-Bin and Zha, Hongyuan},
  booktitle={Proceedings of the AAAI Conference on Artificial Intelligence},
  volume={35},
  number={8},
  pages={7474--7482},
  year={2021}
}

@article{chen2023two,
  title={Two-Stage Learning For the Flexible Job Shop Scheduling Problem},
  author={Chen, Wenbo and Khir, Reem and Van Hentenryck, Pascal},
  journal={arXiv preprint arXiv:2301.09703},
  year={2023}
}

@article{song2022flexible,
  title={Flexible Job-Shop Scheduling via Graph Neural Network and Deep Reinforcement Learning},
  author={Song, Wen and Chen, Xinyang and Li, Qiqiang and Cao, Zhiguang},
  journal={IEEE Transactions on Industrial Informatics},
  volume={19},
  number={2},
  pages={1600--1610},
  year={2022},
  publisher={IEEE}
}

@inproceedings{ioffe2015batch,
  title={Batch normalization: Accelerating deep network training by reducing internal covariate shift},
  author={Ioffe, Sergey and Szegedy, Christian},
  booktitle={International conference on machine learning},
  pages={448--456},
  year={2015},
  organization={pmlr}
}

@article{srivastava2014dropout,
  title={Dropout: a simple way to prevent neural networks from overfitting},
  author={Srivastava, Nitish and Hinton, Geoffrey and Krizhevsky, Alex and Sutskever, Ilya and Salakhutdinov, Ruslan},
  journal={The journal of machine learning research},
  volume={15},
  number={1},
  pages={1929--1958},
  year={2014},
  publisher={JMLR. org}
}

@article{rumelhart1986learning,
  title={Learning representations by back-propagating errors},
  author={Rumelhart, David E and Hinton, Geoffrey E and Williams, Ronald J},
  journal={nature},
  volume={323},
  number={6088},
  pages={533--536},
  year={1986},
  publisher={Nature Publishing Group UK London}
}

@ARTICLE{E2ELR,
  author={Chen, Wenbo and Tanneau, Mathieu and Van Hentenryck, Pascal},
  journal={IEEE Transactions on Power Systems}, 
  title={End-to-End Feasible Optimization Proxies for Large-Scale Economic Dispatch}, 
  year={2024},
  volume={39},
  number={2},
  pages={4723-4734},
  doi={10.1109/TPWRS.2023.3317352}}

@inproceedings{AAAI2020,
  author       = {Ferdinando Fioretto and
                  Terrence W. K. Mak and
                  Pascal Van Hentenryck},
  title        = {Predicting {AC} Optimal Power Flows: Combining Deep Learning and Lagrangian
                  Dual Methods},
  booktitle    = {The Thirty-Fourth {AAAI} Conference on Artificial Intelligence, {AAAI}
                  2020, The Thirty-Second Innovative Applications of Artificial Intelligence
                  Conference, {IAAI} 2020, The Tenth {AAAI} Symposium on Educational
                  Advances in Artificial Intelligence, {EAAI} 2020, New York, NY, USA,
                  February 7-12, 2020},
  pages        = {630--637},
  publisher    = {{AAAI} Press},
  year         = {2020},
  url          = {https://ojs.aaai.org/index.php/AAAI/article/view/5403},
  bibsource    = {dblp computer science bibliography, https://dblp.org}
}

@inproceedings{AAAI2023,
  author    = {Park, Seonho and Van~Hentenryck, Pascal},
  title     = {Self-Supervised Primal-Dual Learning for Constrained Optimization.},
  booktitle = {Proceedings of the 37th {AAAI} Conference on Artificial Intelligence, Washington, DC},
  month     = {February},
  year      = {2023}
}

@ARTICLE{IEEETPS2022spatial,
  author={Chatzos, Minas and Mak, Terrence W. K. and Van~Hentenryck, Pascal},
  journal={IEEE Transactions on Power Systems}, 
  title={Spatial Network Decomposition for Fast and Scalable AC-OPF Learning}, 
  year={2022},
  volume={37},
  number={4},
  pages={2601-2612},
  doi={10.1109/TPWRS.2021.3124726}}

@article{PSCC2022,
title = {Learning optimization proxies for large-scale Security-Constrained Economic Dispatch},
journal = {Electric Power Systems Research},
volume = {213},
pages = {108566},
year = {2022},
issn = {0378-7796},
doi = {https://doi.org/10.1016/j.epsr.2022.108566},
url = {https://www.sciencedirect.com/science/article/pii/S0378779622006629},
author = {Wenbo Chen and Seonho Park and Mathieu Tanneau and Pascal {Van Hentenryck}}
}

@misc{SIAMNEWS, 
title={Fusing artificial intelligence and optimization with trustworthy optimization proxies}, 
url={https://sinews.siam.org/Details-Page/fusing-artificial-intelligence-and-optimization-with-trustworthy-optimization-proxies}, 
journal={SIAM News}, 
author={Van~Hentenryck, Pascal}}
